\documentclass[10pt,twocolumn,letterpaper]{article}

\usepackage[pagenumbers]{cvpr}

\usepackage[most]{tcolorbox}
\usepackage{amsthm}
\usepackage{bm}
\usepackage{multirow}
\usepackage{algorithm}
\usepackage{algpseudocode}
\algtext*{EndIf}
\algtext*{EndFor}
\algtext*{EndWhile}
\usepackage{float}

\DeclareMathOperator*{\argmin}{argmin}

\DeclareMathOperator{\child}{chi}
\DeclareMathOperator{\parent}{par}
\DeclareMathOperator{\desc}{desc}
\DeclareMathOperator{\anc}{anc}

\DeclareMathOperator{\lev}{lev}

\newcommand{\cA}{\ensuremath{\mathcal{A}}}

\newcommand{\cC}{\ensuremath{\mathcal{C}}}

\newcommand{\cH}{\ensuremath{\mathcal{H}}}
\newcommand{\cI}{\ensuremath{\mathcal{I}}}

\newcommand{\cS}{\ensuremath{\mathcal{S}}}

\newcommand{\cV}{\ensuremath{\mathcal{V}}}

\newcommand{\bzero}{\ensuremath{\bm{0}}}

\newcommand{\bA}{\ensuremath{\bm{A}}}

\newcommand{\bD}{\ensuremath{\bm{D}}}

\newcommand{\bU}{\ensuremath{\bm{U}}}

\newcommand{\bW}{\ensuremath{\bm{W}}}

\newcommand{\ba}{\ensuremath{\bm{a}}}

\newcommand{\bd}{\ensuremath{\bm{d}}}
\newcommand{\be}{\ensuremath{\bm{e}}}

\newcommand{\bg}{\ensuremath{\bm{g}}}
\newcommand{\bh}{\ensuremath{\bm{h}}}

\newcommand{\br}{\ensuremath{\bm{r}}}
\newcommand{\bs}{\ensuremath{\bm{s}}}

\newcommand{\bw}{\ensuremath{\bm{w}}}
\newcommand{\bx}{\ensuremath{\bm{x}}}

\newcommand{\bz}{\ensuremath{\bm{z}}}

\def\st/{\textsuperscript{st}}
\def\nd/{\textsuperscript{nd}}
\def\rd/{\textsuperscript{rd}}
\def\th/{\textsuperscript{th}}

\newcommand{\R}{\mathbb{R}}

\usepackage{aliascnt}

\newaliascnt{proposition}{theorem}
\newtheorem{proposition}[proposition]{Proposition}
\aliascntresetthe{proposition}
\newaliascnt{lemma}{theorem}
\newtheorem{lemma}[lemma]{Lemma}
\aliascntresetthe{lemma}
\newaliascnt{corollary}{theorem}

\aliascntresetthe{corollary}
\theoremstyle{definition}
\newaliascnt{definition}{theorem}
\newtheorem{definition}[definition]{Definition}
\aliascntresetthe{definition}
\newaliascnt{assumption}{theorem}

\aliascntresetthe{assumption}
\theoremstyle{remark}
\newaliascnt{remark}{theorem}

\aliascntresetthe{remark}

\newcommand{\myparagraph}[1]{\smallskip\noindent\textbf{#1.}}

\definecolor{cvprblue}{rgb}{0.21,0.49,0.74}
\usepackage[pagebackref,breaklinks,colorlinks,allcolors=cvprblue]{hyperref}
\graphicspath{{./}}

\def\paperID{}
\def\confName{CVPR}
\def\confYear{2026}

\title{\centerline{Hierarchical Concept Embedding \& Pursuit for Interpretable Image Classification}}

\newcommand{\ourslong}{Hierarchical Concept Embedding \& Pursuit}
\newcommand{\ours}{HCEP}

\author{\textbf{Nghia Nguyen\thanks{Preprint. Correspondence to \href{mailto:nghianhh@seas.upenn.edu}{\nolinkurl{nghianhh@seas.upenn.edu}}}\quad Tianjiao Ding\quad René Vidal}\\
University of Pennsylvania\\
}

\begin{document}
\maketitle
\begin{abstract}
Interpretable-by-design models are gaining traction in computer vision because they provide faithful explanations for their predictions. In image classification, these models typically recover human-interpretable concepts from an image and use them for classification. Sparse concept recovery methods leverage the latent space of vision-language models to represent image embeddings as sparse combinations of concept embeddings. However, by ignoring the hierarchical structure of semantic concepts, these methods may produce correct predictions with explanations that are inconsistent with the hierarchy. In this work, we propose \ourslong{} (\ours{}), a framework that induces a hierarchy of concept embeddings in the latent space and performs hierarchical sparse coding to recover the concepts present in an image. Given a hierarchy of semantic concepts, we introduce a geometric construction for the corresponding hierarchy of embeddings. Under the assumption that the true concepts form a rooted path in the hierarchy, we derive sufficient conditions for their recovery in the embedding space. We further show that hierarchical sparse coding reliably recovers hierarchical concept embeddings, whereas standard sparse coding fails. Experiments on real-world datasets show that \ours{} improves concept precision and recall compared to existing methods while maintaining competitive classification accuracy. Moreover, when the number of samples available for concept estimation and classifier training is limited, \ours{} achieves superior classification accuracy and concept recovery. Our results demonstrate that incorporating hierarchical structure into sparse concept recovery leads to more faithful and interpretable image classification models.\footnote{Code at \url{https://github.com/nghiahhnguyen/hcep}.}

\end{abstract}

\section{Introduction}
Machine learning has been adopted in many computer vision applications, including image classification, question answering, and concept recovery, often matching or surpassing human performance \citep{krizhevsky2012imagenet,radford_learning_2021}. However, the lack of interpretability in these models' predictions has raised concerns about their trustworthiness \citep{doshi-velez_towards_2017,rudin_stop_2019,lipton_mythos_2018}. 

\begin{figure}[t]
    \centering
    \includegraphics[width=\linewidth]{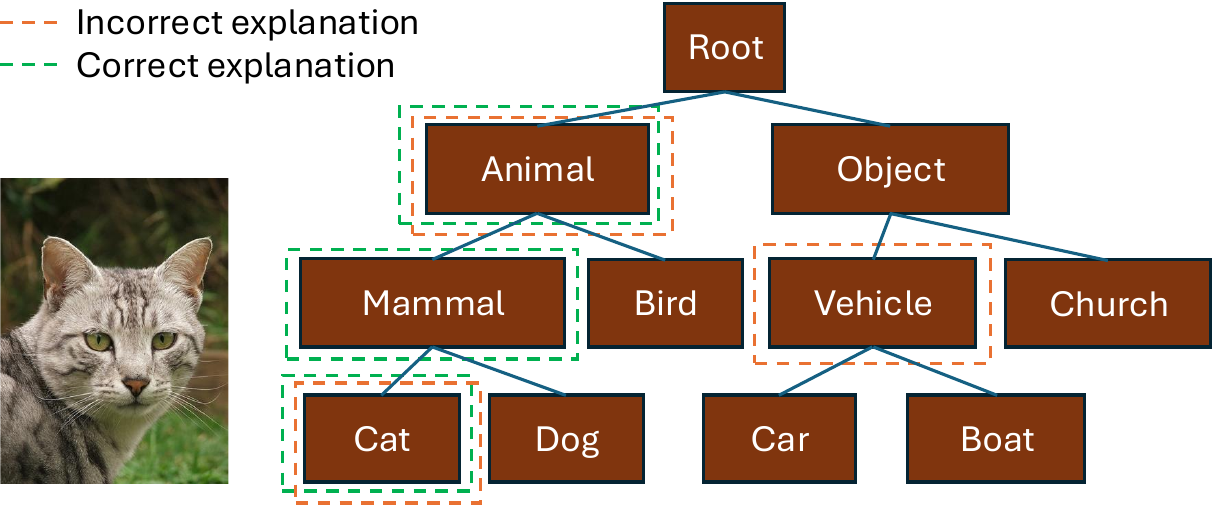}
    \vspace{-0.5cm}
    \caption{
        \textbf{An illustration of hierarchical concept explanations for image classification.} Given an image of a \texttt{cat}, an interpretable-by-design model should extract the concepts that form a rooted path in the hierarchy, and use only these concepts to classify the image as a \texttt{cat}. However, existing concept recovery methods may extract concepts that are inconsistent with the hierarchy, such as \texttt{vehicle}, leading to incorrect explanations.
    }
    \vspace{-0.3cm}
    \label{fig:main}
\end{figure}

Interpretable models in computer vision generally fall into two categories: \textit{post-hoc explanation} and \textit{interpretable-by-design} methods.
Post-hoc explanation methods aim to provide insights into the decision-making process of trained black-box models~\citep{ribeiro_why_2016,lundberg_unified_2017,selvaraju_gradcam_2017,sundararajan_axiomatic_2017}. However, such methods often suffer from a lack of faithfulness and stability with respect to the original pre-trained models~\citep{adebayo_sanity_2018,bilodeau_impossibility_2024}. This paper focuses on interpretable-by-design methods, which build interpretability directly into the model training process~\citep{chen_prototype_2019,alvarez-melis_towards_2018,koh2020concept}. Such models usually consist of two steps: (1) extracting human-interpretable concepts from the input and (2) using only these concepts for downstream tasks such as classification or regression.\footnote{
    There is an emerging use of Chain-of-Thought~\citep{wei2022chain} as an interpretability method, which is not fully interpretable-by-design; see \S\ref{sec:interp_design_vs_cot}.
} 
For instance, given an image of a cat, an interpretable-by-design model would extract concepts such as \texttt{animal}, \texttt{furry}, and \texttt{short muzzle}, and then classify the image using only these concepts. 

Recent methods for recovering human-interpretable concepts from images differ in how they use supervision and whether they can handle unseen concepts. 
In \textit{fully supervised concept recovery}, a model is trained to predict a predefined set of concepts \citep{koh2020concept, kim2018interpretability}. While effective, this approach does not generalize to unseen concepts and requires annotations that are costly to obtain for large datasets with many concepts.
In \textit{concept-specific supervised recovery}, a model is trained to predict whether a given concept is present in a given image. While this approach also requires annotated data, by leveraging vision-language embeddings such as CLIP \cite{radford_learning_2021}, it can generalize to new concepts at inference time \cite{chattopadhyay2024bootstrapping}.
\textit{Zero-shot concept recovery} methods rely solely on pre-trained vision-language models to predict concepts without requiring additional annotations \citep{menon_visual_2022}, albeit at the cost of runtime efficiency.
Finally, \textit{sparse concept recovery} methods start with a predefined set of concepts and identify those present in an image by representing the image embedding as a sparse linear combination of the concept embeddings \cite{chattopadhyay_information_2023,bhalla_interpreting_2024}. This approach is particularly scalable because it focuses on selecting a small number concepts.  

However, the aforementioned methods fail to capture hierarchical relationships among concepts, such as those between hypernyms and hyponyms. 
For instance, in \cref{fig:main}, the concept hierarchy takes the form of a tree whose leaves correspond to the object classes.
In this case, there is a unique path connecting each class to the root of the tree, and the concepts along that path can be viewed as an explanation for the class.
{\color{black}For example, the explanation for \texttt{cat} in \Cref{fig:main} is given by the path \texttt{animal} $\rightarrow$ \texttt{mammal} $\rightarrow$ \texttt{cat}.}
By neglecting this hierarchy, existing concept recovery methods may identify concepts that are inconsistent with it, leading to unreliable explanations and predictions. We address this limitation by drawing inspiration from prior work on the geometry of hierarchical concepts in large language models \citep{park_geometry_2025} and from hierarchical sparse coding \citep{jost_tree-based_2006, la_tree-based_2006,jenatton2011proximal}. This leads to a hierarchical sparse coding formulation for concept recovery that produces more reliable explanations.

\myparagraph{Contributions} In this paper, we propose \ourslong{} (\ours), a framework that leverages hierarchical relationships between concepts to improve concept recovery for interpretable image classification. We summarize \ours{} below and highlight our contributions.

\begin{itemize}
    \item \textit{Hierarchical Concept Embedding:} In \Cref{sec:generative_model}, we propose a geometric framework for embedding a given hierarchy of semantic concepts. More specifically, we identify the following \textit{ideal geometric properties} for hierarchical concept embeddings: embeddings of descendant synsets\footnote{A synset is a categorical concept that groups all synonyms together (e.g., \texttt{cats}). A concept is a higher level abstraction that includes synsets and differences of synsets. One can view synsets as nodes in a hierarchy and synset differences as edges.} 
    should be close to those of their parent synsets, while embeddings of sibling synsets should be well separated. We also incorporate geometric conditions inspired by prior work on the hierarchical geometry of large language models \citep{park_geometry_2025}. We analyze these properties theoretically and verify them empirically in vision-language models. 
    \item \textit{Hierarchical Concept Pursuit:} In \Cref{sec:sparse_recovery_hierarchical}, we leverage the geometric properties of hierarchical concept embeddings to propose a concept recovery procedure that proceeds in two steps. First, it constructs a \textit{hierarchical dictionary} from pre-trained vision-language embeddings by taking differences between the embeddings of parent and child synsets. Second, it leverages \textit{hierarchical sparse coding} to recover a rooted path in the hierarchy, thereby recovering concepts for interpretable classification. 
    \item \textit{Interpretable Image Classification:} In \Cref{sec:experiments}, we demonstrate through experiments on synthetic and real-world datasets that \ours{} outperforms sparse concept recovery baselines in concept recovery while maintaining competitive classification accuracy. In few-shot settings, \ours{} outperforms all interpretable baselines in both classification accuracy and concept recovery. We also show that for datasets without a predefined hierarchy, we can construct a meaningful hierarchy using taxonomy induction methods \citep{hearst_automatic_1992,le_inferring_2019,zeng2024chain} and still achieve improved concept recovery using our framework.
\end{itemize}

\section{Preliminaries}\label{sec:preliminaries}

\subsection{Interpretable-by-design classification models}

Interpretable-by-design classification models often proceed in two steps: (1) predicting human-interpretable concepts from the input and (2) using this concept-based representation for classification. Depending on the approach, interpretability may arise either from the use of linear classifiers (e.g., in concept bottleneck models (CBMs) \citep{koh2020concept} the importance of a concept depends on the magnitude of the classifier coefficients) or from the selection of a small set of most relevant concepts for classification (e.g., via information gain \cite{chattopadhyay2023variational} or semantic sparse coding \cite{chattopadhyay_information_2023,bhalla_interpreting_2024}). As a result, the main design space in these models lies in reliable concept prediction. CBMs use supervision to learn concept predictors; however, they are not scalable because they require additional labeled data. Semantic sparse coding approaches instead use pre-trained embeddings to extract concepts from data, with the goal of representing an image embedding as a sparse linear combination of concept embeddings. In this work, we focus on extracting concepts from pre-trained image embeddings via a sparse coding objective.

\subsection{Sparse coding for concept extraction}\label{sec:preliminaries_sparse_coding}
Sparse coding aims to represent data as a sparse linear combination of basis elements, typically chosen from an overcomplete dictionary \cite{foucart2013invitation}. Given an input signal $\mathbf{x} \in \mathbb{R}^d$ and a dictionary $\mathbf{D} \in \mathbb{R}^{d \times k}$, sparse coding seeks a sparse vector $\mathbf{z} \in \mathbb{R}^k$ such that $\mathbf{x} \approx \mathbf{D} \mathbf{z}$, where the sparsity of $\mathbf{z}$ encourages the model to use only a few dictionary elements to reconstruct the input signal. 
To find a sparse representation $\mathbf{z}$, we can solve the following optimization problem:
\begin{equation}\label{eq:sparse_coding}
    \min_{\mathbf{z}} \|\mathbf{x} - \mathbf{D} \mathbf{z}\|_2^2 + \lambda \|\mathbf{z}\|_0,
\end{equation}
where $\|\cdot\|_0$ denotes the $\ell_0$ semi-norm and $\lambda$ is a regularization parameter that controls the trade-off between reconstruction accuracy and sparsity.
The optimization problem in \eqref{eq:sparse_coding} can be solved using various algorithms, such as orthogonal matching pursuit (OMP) \cite{pati1993orthogonal} or basis pursuit ($\ell_1$ relaxation) \cite{chen2001atomic}. Intuitively, standard OMP finds a sparse solution greedily: starting from the input $\bx$ as the residual error, it repeatedly selects the dictionary atom most aligned with the current residual, adds that atom to the support, recomputes the coefficients of the selected atoms by least squares, and then updates the residual by subtracting its projection onto the span of the selected atoms (see \Cref{alg:omp}). 

In interpretable image classification, the signal $\mathbf{x}$ is typically the embedding of an image obtained from a pre-trained model, the dictionary $\mathbf{D}$ consists of text embeddings corresponding to human-interpretable concepts, and the magnitudes of the sparse coefficients in $\bz$ indicate the importance of the corresponding concepts for representing the image \cite{chattopadhyay_information_2023,bhalla_interpreting_2024,gandelsman_interpreting_2024-1}.\footnote{Although not the initial motivation, the validity of this approach is supported by the linear representation and superposition hypotheses \cite{elhage2023superposition,park_linear_2024,saglamLargeLanguageModels2025}.}
Thus, OMP can be viewed as iteratively asking which concept (atom) best explains what is left unexplained by the current explanation (residual)
\cite{chattopadhyay_information_2023}.

However, sparse coding treats all columns of $\bD$ equally, and thus it neglects the hierarchical structure of the concept set. For example, the synset \texttt{vehicle} can be further divided into \texttt{car}, \texttt{truck}, and \texttt{motorcycle}, each of which can be further divided into more specific synsets. These concepts, which \textit{implicitly} describe the difference between a fine-grained synset and its parent synset, correspond to the edges of the synset hierarchy. To capture this hierarchical structure, we need to extend the sparse coding framework to incorporate hierarchical relationships among concepts. As motivation for this goal, we next briefly discuss related work on geometric structure in vector embeddings.

\subsection{Geometric Structures of Meanings in Vector Embeddings}\label{sec:preliminaries_geometric_structures}
A notable example of geometric structure in vector embeddings is Word2Vec \cite{mikolov2013distributed}, where contrastive semantic relationships can be captured through vector arithmetic. More recent work has explored linear and compositional structures in a vector space \citep{trager2023linear,trager2024compositional,park_linear_2024,jiang2024origins}. Hierarchical compositional structures have also been studied for learning part-based representations \citep{kortylewski_greedy_2019,sclocchi2024probing}.
On the other hand, \citet{park_geometry_2025} connect hierarchical class-based semantics to a geometric condition: making a concept more specific should not interfere with the coarser concept. This intuition can be written as
\begin{align}
  \left(\mathbf{v}_{\mathrm{child}} - \mathbf{v}_{\mathrm{parent}}\right)^\top \mathbf{v}_{\mathrm{parent}} = 0.
\end{align}
where $\mathbf{v}_{\mathrm{parent}}$ and $\mathbf{v}_{\mathrm{child}}$ denote representations along one branch of a hierarchy. In the next section, we build on this insight to design a hierarchical concept embedding model.

\section{Hierarchical Concept Embedding}\label{sec:generative_model}

Given a hierarchy of semantic concepts, we seek vector representations whose geometry facilitates hierarchical concept recovery. In this section, we identify ideal geometric conditions that hierarchical concept embeddings should satisfy to enable such recovery, and later show that these conditions are empirically supported in vision-language models. The specific conditions we consider are the following:
\begin{itemize}
    \item \textit{Well-clustered synset embeddings (\Cref{sec:well_clustered_synsets}):} In the embedding space, sibling synsets should cluster sufficiently tightly around their parent, while remaining sufficiently well separated from one another to be easily distinguished.
    As a consequence, concepts with different ancestries are well separated and can be reliably recovered.
    \item \textit{Hierarchical orthogonality (\Cref{sec:hierarchical_independence}):} Given a parent node, the difference between a child's embedding and the parent's embedding should be orthogonal to the parent's embedding. This condition helps preserve semantic meaning across levels of the hierarchy in the embedding space. 
\end{itemize}
\Cref{fig:data_generating_process} provides a two-dimensional illustration of the desired geometry: descendants cluster around their parents, while child-parent difference vectors capture refinements that are orthogonal to the parent's representation. These conditions drive the design of the concept recovery algorithm in \Cref{sec:sparse_recovery_hierarchical}.

\begin{figure}[tbp]
  \centering
  \includegraphics[width=\linewidth]{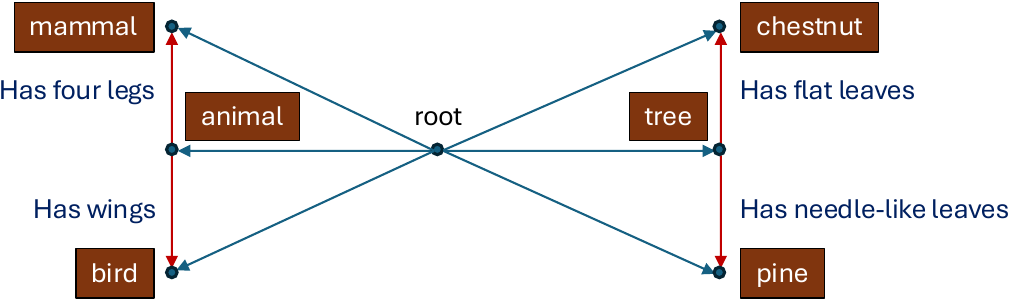}
  \vspace{-0.4cm}
  \caption{\textbf{Illustration of the hierarchical data model in $\mathbb{R}^2$.} The root node has two child nodes, \texttt{animal} and \texttt{tree}, each with its own children. There are two desirable conditions for concept identifiability: (1) children cluster around the parent while sibling nodes remain separated; and (2) the difference between a child and its parent (shown as red arrows) is orthogonal to the parent, and the differences between the children of a parent form a simplex (which is a line in $\mathbb{R}^2$). The difference vectors capture the characteristics that distinguish each child from its parent. See \Cref{sec:generative_model} for more details.}
  \label{fig:data_generating_process}
  \vspace{-0.3cm}
\end{figure}

\begin{figure}[tb]
  \centering
  \includegraphics[width=\linewidth]{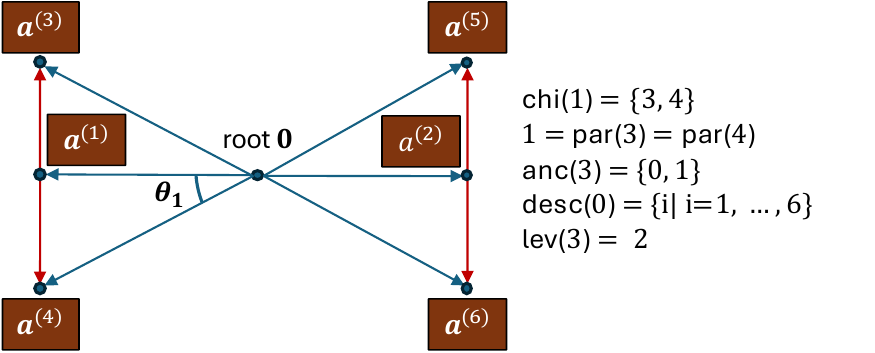}
  \vspace{-0.5cm}
  \caption{A hierarchy with $L=2$ levels, branching ratio $b=2$, and $N_L=6$ nodes.}
  \label{fig:notations}
  \vspace{-0.4cm}
\end{figure}

\subsection{Well-clustered synset embeddings}\label{sec:well_clustered_synsets}

To ensure that each node in the hierarchy can be uniquely assigned to its parent, we impose geometric constraints on the embedding space. Specifically, consider a hierarchy of synsets, such as the one illustrated in \Cref{fig:notations}. Let $0$ denote the root node and $\mathcal{A} = \{1, \ldots, N_L\}$ be the set of non-root nodes. Assume the hierarchy has $L$ levels and that each internal node has $b$ children. For node $i$, let $\parent(i)$, $\child(i)$, $\anc(i)$, and $\desc(i)$ denote its parent, children, ancestors, and descendants, respectively. The level of node $i$ is the number of its ancestors, namely $\lev(i)=|\anc(i)|$. Our goal is to associate each node $i$ with a vector representation $\ba^{(i)} \in \mathbb{R}^d$.
The following proposition formalizes conditions on $\{\ba^{(i)}\}_{i=1}^{N_L}$ that guarantee well-separated subtrees.

\begin{proposition}[Well-clustered hierarchy ensures unique parent assignment]
\label{prop:well_clustered_hierarchy}
Suppose the following geometric conditions hold for all nodes in the hierarchy:
\begin{enumerate}
    \item \textbf{Subtree containment:}
    The vector representations of the descendants of node $i$, $\{\ba^{(j)}\}_{j \in \desc(i)}$, are contained in a cone with axis $\ba^{(i)}$ and half-angle $\theta_{\lev(i)}$, i.e.,
    \begin{equation}\label{eq:subtree_well_clustered}
      \forall i \in \cA, \quad \max_{j \in \desc(i)} \angle(\ba^{(i)}, \ba^{(j)}) \leq \theta_{\lev(i)}.
    \end{equation}

    \item \textbf{Sibling-cone disjointness:}
    For any parent node $i$ and any pair of distinct children $j, j' \in \child(i)$, their corresponding subtree cones do not intersect:
    \begin{equation}\label{eq:cone_disjoint}
      \angle(\ba^{(j)}, \ba^{(j')}) > \theta_{\lev(j)} + \theta_{\lev(j')}.
    \end{equation}
\end{enumerate}
Then the subtrees rooted at any two sibling nodes are disjoint, and every node has a unique parent.
\end{proposition}

All proofs are provided in \Cref{app:proofs}. A sufficient way to satisfy both conditions in \Cref{prop:well_clustered_hierarchy} is to impose a geometrically decreasing half-angle schedule, as presented next.

\begin{proposition}[Geometrically decreasing half-angle schedule]\label{prop:angle_schedule}
If the half-angles satisfy
\begin{equation}
  	\theta_{l+1} \le \min\{r, 1/b\}\,\theta_l
\end{equation}
for some $r\in(0,1/2)$, there exists a placement of node embeddings such that the conditions in \Cref{prop:well_clustered_hierarchy} are satisfied.
\end{proposition}

\subsection{Hierarchical Orthogonality}\label{sec:hierarchical_independence}

Inspired by the geometric conditions on concept embeddings in large language models \citep{park_geometry_2025}, we further impose hierarchical orthogonality and simplex conditions on the concept embeddings. 
\begin{itemize}
    \item A parent embedding is orthogonal to the child-parent difference vector. That is, for a parent node $i$, any child node $j \in \child(i)$ satisfies $(\ba^{(j)} - \ba^{(i)})^\top \ba^{(i)} = 0$. 
    \item The differences between the embeddings of two children of node $i$, $\{\ba^{(j)} - \ba^{(j')}\}_{j, j' \in \child(i)}$, form a $(b-1)$-simplex.
\end{itemize}
As we shall see in the next section, these conditions will guide the construction of the dictionary we will use for hierarchical sparse coding. However, for them to hold, the embedding space must satisfy a minimum dimension requirement, as stated in the following proposition.
\begin{proposition}[Depth--branch--dimension necessity]\label{prop:depth_dimension}
In a hierarchy with depth $L$ and branching ratio $b$, suppose all nodes satisfy hierarchical orthogonality and all sibling nodes form a regular $(b-1)$-simplex as in \eqref{eq:multi_ortho} and \eqref{eq:multi_simplex_ip}. Then the embedding dimension must satisfy $d \ge L + b - 1$.
\end{proposition}
Intuitively, at depth $l\in [L-1]$, hierarchical orthogonality imposes $l+1$ independent affine constraints, restricting children to a $(d-l-1)$-dimensional feasible subspace. Embedding a regular $(b-1)$-simplex within that subspace requires $(d-l-1) \ge (b-1)$. Evaluating this at the deepest internal level $l=L-1$ yields $d \ge L + b - 1$. This condition is satisfied in practice. For example, CLIP embeddings have $d=768 \gg 38$ ($14 + 25 - 1$ for WordNet's depth and maximum branching ratio).

\section{Hierarchical Concept Pursuit}\label{sec:sparse_recovery_hierarchical}

Given the concept embedding model described in \Cref{sec:generative_model}, we now turn to the problem of recovering the sparse representation of a signal generated from this model. To do so, we first construct a hierarchical dictionary that captures the hierarchical structure of the synsets and concepts (\Cref{sec:hierarchical_dictionary_construction}). We then propose a hierarchical sparse coding algorithm that leverages this structure to improve sparse recovery (\Cref{sec:hierarchical_omp}).

\subsection{Hierarchical Concept Dictionary Construction}\label{sec:hierarchical_dictionary_construction}

First, we define the hierarchical dictionary $\bD$ as
\begin{equation}\label{eq:full_dictionary}
\bD = \bigl[\, \ba^{(j)} - \ba^{(\parent(j))} \,\bigr]_{j \in \cA}.
\end{equation}
Each column captures the difference between a synset and its parent. 
We define $\ba^{(\mathrm{root})} = \mathbf{0}$, so the atoms for root children reduce to $\ba^{(j)} - \mathbf{0} = \ba^{(j)}$.

In the context of interpretable image classification, $\ba^{(i)}$ represents the embedding of a synset within a concept hierarchy, such as WordNet \citep{miller_wordnet_1995}, while the difference $\ba^{(i)} - \ba^{(\parent(i))}$ represents the concept that distinguishes the synset from its parent. For example, let node $i$ denote the synset \texttt{bear} and let node $j \in \child(i)$ denote the synset \texttt{polar bear}. Then the difference $\ba^{(j)} - \ba^{(i)}$ can be interpreted as the refinement that distinguishes \texttt{polar bear} from \texttt{bear}, such as \texttt{white fur}. This construction of the dictionary in \eqref{eq:full_dictionary} avoids trivial solutions in the sparse formulation of concept recovery; for example, the sparsest explanation for an image of a \texttt{cat} would otherwise be simply the concept \texttt{cat}. Note that the difference embeddings have a grounded and interpretable meaning on their own, as they represent directions that differentiate child synsets from their parents. 
An illustration of this hierarchical sparse decomposition is shown in \Cref{fig:hierarchical_explanation}.

\begin{figure}[t]
  \centering
  \includegraphics[width=\linewidth]{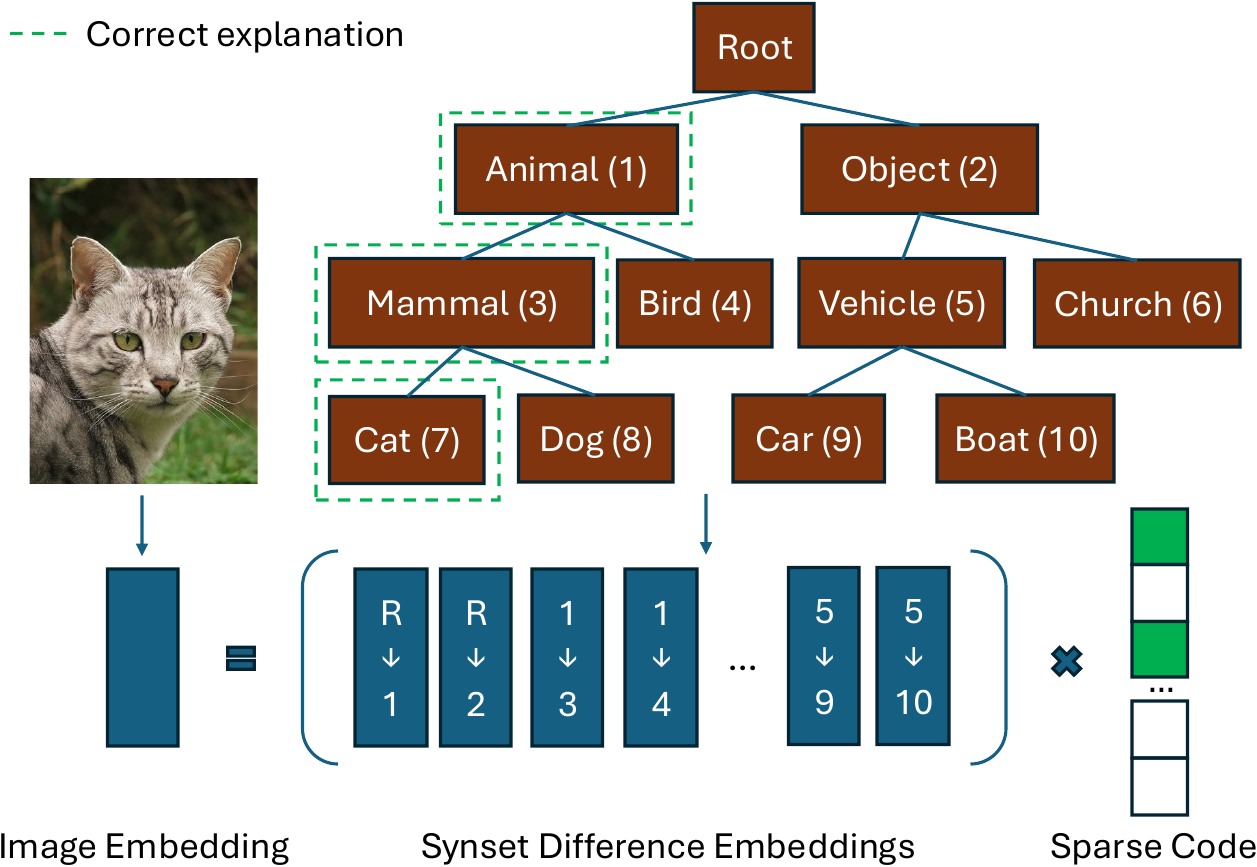}
  \vspace{-0.5cm}
  \caption{\textbf{Illustration of the hierarchical sparse decomposition in \eqref{eq:noisy_hier_sparse_expansion}.} Given an image of a cat, the correct explanation follows the path from the root to the leaf node \texttt{Cat} (shown in green dashed lines). The image embedding is expressed as a sparse linear combination of the root child synset embedding (e.g., \texttt{Animal}) and synset difference embeddings along the path (e.g., \texttt{Mammal} $-$ \texttt{Animal}, \texttt{Cat} $-$ \texttt{Mammal}). The sparse code has non-zero entries only for atoms corresponding to nodes on the correct path.}
  \label{fig:hierarchical_explanation}
  \vspace{-0.3cm}
\end{figure}

Under this hierarchical dictionary, a synset embedding can be written as a sparse linear combination of atoms along the root-to-node path:
\begin{equation}\label{eq:hier_sparse_expansion}
  \ba^{(i)}
  = \sum_{j \in \anc(i) \cup \{i\}}
      \bigl(\ba^{(j)} - \ba^{(\parent(j))}\bigr),
\end{equation}
which follows by telescoping. Since $\ba^{(\mathrm{root})} = \mathbf{0}$, the coefficients along the path are all equal to $1$ in the noiseless case. If an image embedding $\bx$ is generated from synset $i$ with additive noise, i.e., $\bx = \ba^{(i)} + \epsilon$, then we expect a noisy expansion with coefficients $\alpha_j \in \mathbb{R}$:
\begin{equation}\label{eq:noisy_hier_sparse_expansion}
  \bx = \sum_{j \in \anc(i) \cup \{i\}} \alpha_j
      \bigl(\ba^{(j)} - \ba^{(\parent(j))}\bigr).
\end{equation}
Recovering the synset $i$ is therefore equivalent to recovering a sparse code whose support corresponds to the nodes on the root-to-$i$ path, although the coefficients may deviate from $1$.

\subsection{Hierarchical Beam OMP (HB-OMP)}
\label{sec:hierarchical_omp}

\definecolor{ModAColor}{RGB}{0,92,184}
\definecolor{ModBColor}{RGB}{176,40,120}
\definecolor{ModAFill}{RGB}{230,245,255}
\definecolor{ModBFill}{RGB}{250,234,246}
\definecolor{CommentGray}{RGB}{128,128,128}
\newcommand{\ModA}[1]{{\color{black}\textcolor{ModAColor}{#1}\color{black}}}
\newcommand{\ModB}[1]{{\color{black}\textcolor{ModBColor}{#1}\color{black}}}
\newcommand{\ModAline}[1]{{\color{black}\colorbox{ModAFill}{\color{black}#1}\color{black}}}
\newcommand{\ModBline}[1]{{\color{black}\colorbox{ModBFill}{\color{black}#1}\color{black}}}
\algrenewcommand{\algorithmiccomment}[1]{\hfill{\color{CommentGray}$\triangleright$ #1}}

\begin{algorithm}[t]
  \caption{Orthogonal Matching Pursuit (OMP)}
  \label{alg:omp}
  \begin{algorithmic}[1]
  \Require $\bx \in \R^d$, dict $\bD$, tol $\epsilon$, max steps $T$
  \State Init residual error $\br^{(0)} \gets \bx$, sparse code $\bz^{(0)} \gets \bzero$
  \For{$t = 1, \ldots, T$}
      \If{$\|\br^{(t-1)}\|_2 < \epsilon$}
          \State \textbf{break}
      \EndIf
      \State $i^\star \gets \arg\max_i \left|\frac{\langle \br^{(t-1)}, \bd^{(i)} \rangle}{\|\br^{(t-1)}\|_2\,\|\bd^{(i)}\|_2}\right|$
      \State $\cS^{(t)} \gets \mathrm{supp}(\bz^{(t-1)}) \cup \{i^\star\}$ \Comment{extend support}
      \State $\bz^{(t)} \gets \argmin\limits_{\bw:\mathrm{supp}(\bw)\subseteq \cS^{(t)}} \|\bx - \bD \bw\|_2^2$ \Comment{update code}
      \State $\br^{(t)} \gets \bx - \bD \bz^{(t)}$ \Comment{update residual}
  \EndFor
  \State Return $\bz^{(t^\star)}$ where $t^\star \in \arg\min_t \|\br^{(t)}\|_2$
  \end{algorithmic}
\end{algorithm}

\begin{algorithm}[t]
  \caption{Hierarchical Beam OMP (HB-OMP)}
  \label{alg:beam-path-hierarchical-omp}
  \begin{algorithmic}[1]
  \Require $\bx \in \R^d$, dict $\bD$, tol $\epsilon$, max steps $T$, children map $\child(\cdot)$, beam size $B$, residual error vector $\br_h$ for hypothesis $h$
  \State Init with a null hypothesis set: $\cH^{(0)} \gets \Big\{\big(\bzero, \bx, \mathrm{root} \big)\Big\}$
  \Comment{each hypothesis consists of (sparse code, residual, last node index)}
  \For{$t = 1, \ldots, T$}
      \If{$\min_{h \in \cH^{(t-1)}} \|\br_h\|_2 < \epsilon$}
          \State \textbf{break} \Comment{some hypothesis have a small residual}
      \EndIf
      \State $\cH_{\text{new}} \gets \emptyset$
      \For{\ModBline{$h=(\bz, \br, i_{last})$ in $\cH^{(t-1)}$}}
            \State \ModAline{$\cI_{\text{active}} \gets \child(i_{last})$}
            \State $c_i \gets \left|\frac{\langle \br, \bd^{(i)} \rangle}{\|\br\|_2\,\|\bd^{(i)}\|_2}\right|$ for all \ModBline{$i \in \cI_{\text{active}}$}
            \State \ModBline{$\cC \gets \text{top-}B$ indices of $c_i$ in $\cI_{\text{active}}$} 
            \State \ModBline{$\cH_{\text{new}} \gets \cH_{\text{new}} \cup \textsc{ExtendHypo}(h, \cC, \bx, \bD)$}
            \Statex \hfill\Comment{\Cref{alg:extend-hypothesis}}
          \EndFor
          \State \ModBline{$\cH^{(t)} \gets$ $B$ lowest-$\|\br_h\|_2^2$ hypotheses in $\cH_{\text{new}}$}
  \EndFor
  \State Return $\bz_{h^\star}$ where $h^\star \in \arg\min_{h \in \cH^{(t)}} \|\br_h\|_2$
  \end{algorithmic}
\end{algorithm}

\begin{algorithm}[t]
  \caption{Extend Hypothesis with Sparse Update}
  \label{alg:extend-hypothesis}
  \begin{algorithmic}[1]
    \Require hypothesis $h=(\bz, \br, i_{last})$, candidates $\cC$, signal $\bx$, dict $\bD$
    \State $\cH_{\text{new}} \gets \emptyset$
    \If{$\cC = \emptyset$} \Comment{leaf reached because no children}
      \State \ModBline{Add $h$ to $\cH_{\text{new}}$} \Comment{keep hypothesis}
      \State \textbf{return} $\cH_{\text{new}}$
    \EndIf
    \For{\ModBline{each $i \in \cC$}}
      \State $\cS' \gets \mathrm{supp}(\bz) \cup \{i\}$ \Comment{extend support}
      \State $\bz' \gets \argmin\limits_{\bw:\mathrm{supp}(\bw)\subseteq \cS'} \|\bx - \bD \bw\|_2^2$ \Comment{update code}
      \State $\br' \gets \bx - \bD \bz'$ \Comment{update residual}
      \State \ModBline{Add $(\bz', \br', i)$ to $\cH_{\text{new}}$}
    \EndFor
    \State \textbf{return} $\cH_{\text{new}}$
  \end{algorithmic}
\end{algorithm}

Given the hierarchical dictionary in \eqref{eq:full_dictionary}, our goal is to recover the sparse code $\bz$ whose support corresponds to the root-to-leaf path of the generating synset \eqref{eq:hier_sparse_expansion}.
As discussed in \Cref{sec:preliminaries}, standard OMP greedily selects the atom most aligned with the current residual and then recomputes the active coefficients. 
If OMP were applied to our hierarchical dictionary, it would treat all atoms as if they were at the same level of the hierarchy, selecting the most correlated atom at each step without respecting the hierarchy.
Therefore, OMP may select atoms that violate the hierarchical structure (e.g., choosing both \texttt{animal} and \texttt{vehicle} at the same level).

In this work, we exploit the hierarchical structure with a coarse-to-fine search: at each step, we consider only children of the current node, descending one level at a time.
This yields Hierarchical Beam OMP (\Cref{alg:beam-path-hierarchical-omp}), a variant of Hierarchical OMP \cite{jost_tree-based_2006,la_tree-based_2006} augmented  with beam search. HB-OMP maintains a \ModB{set} of \ModA{hierarchically valid} support \ModB{hypotheses}, each a partial root-to-node path, and iteratively extends and prunes them according to reconstruction error. Relative to the standard OMP procedure introduced in \Cref{sec:preliminaries_sparse_coding} and detailed in \Cref{alg:omp}, there are two key modifications:

\begin{enumerate}
    \item \ModA{Path-restricted extension:} only children of the deepest explored node are eligible for selection in \Cref{alg:beam-path-hierarchical-omp}, restricting the search to hierarchically valid supports.
    \item \ModB{Beam search:} the algorithm maintains the top-$B$ hypotheses ranked by residual norm, mitigating error accumulation.
An incorrect early decision (e.g., \texttt{animal} vs.\ \texttt{object}) does not propagate to all lower levels.
\end{enumerate}

\begin{figure}[t]
  \centering
  \includegraphics[width=\linewidth]{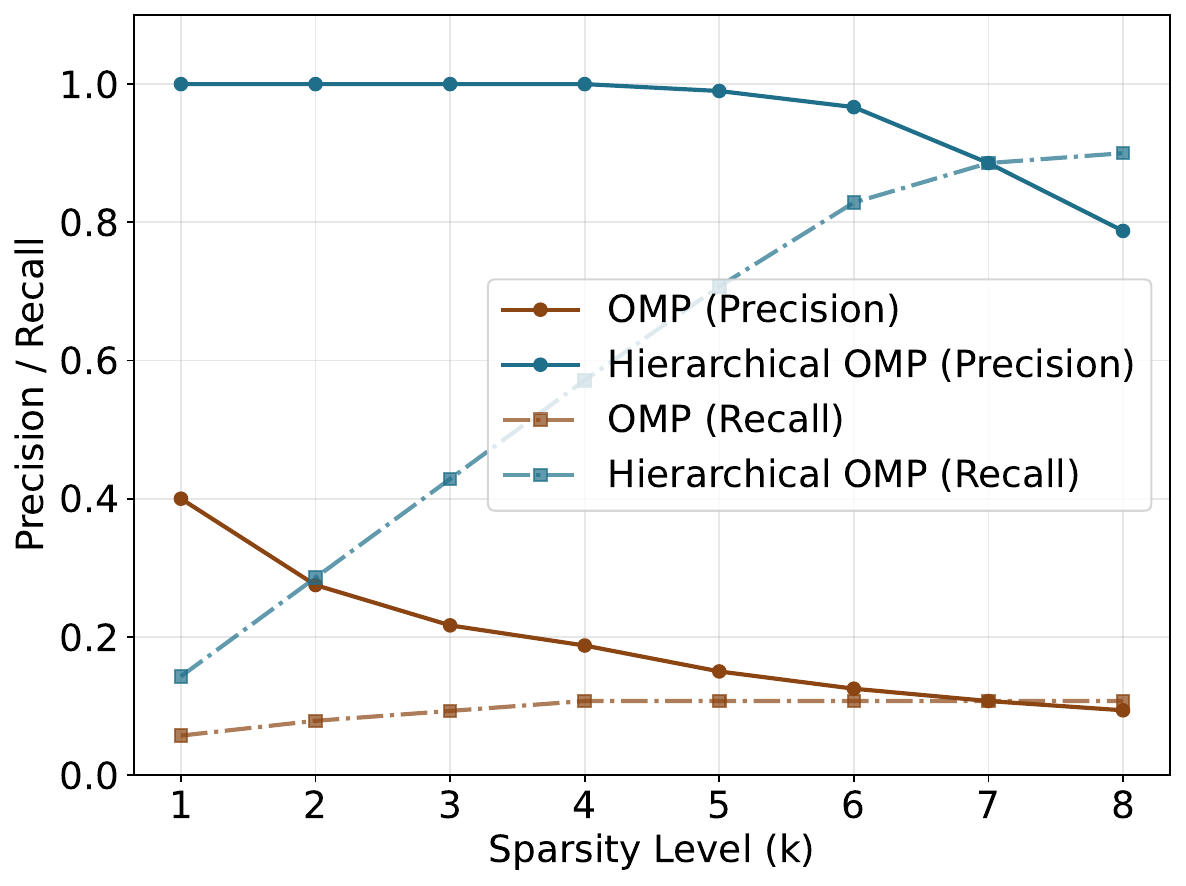}
  \vspace{-0.7cm}
  \caption{Hierarchical OMP has improved support recovery precision and recall compared to standard sparse coding methods on synthetic data.}
  \label{fig:precision_recall_index}
  \vspace{-0.4cm}
\end{figure}

We next give an informal proposition explaining the source of this support-recovery advantage.
See the formal statement and proof in \Cref{app:proof-hier-omp-erc}.

\begin{proposition}[Informal]\label{prop:hier_omp_erc}
Suppose that, at each iteration, HB-OMP extends a hypothesis whose support is a prefix of the true root-to-leaf path. Then its next selection is less likely to introduce an atom outside the true hierarchical support than OMP's selection.
\end{proposition}

This result supports our choice of using beam search to explore multiple hypotheses, as it increases the likelihood of recovering the correct hierarchical support. Beam search also mitigates error accumulation in the top-down search: by maintaining multiple hypotheses at each level, an incorrect early decision (e.g., \texttt{animal} vs. \texttt{object}) does not necessarily propagate to all lower-level explanations.

\begin{figure}[t]
  \centering
  \includegraphics[width=0.48\linewidth]{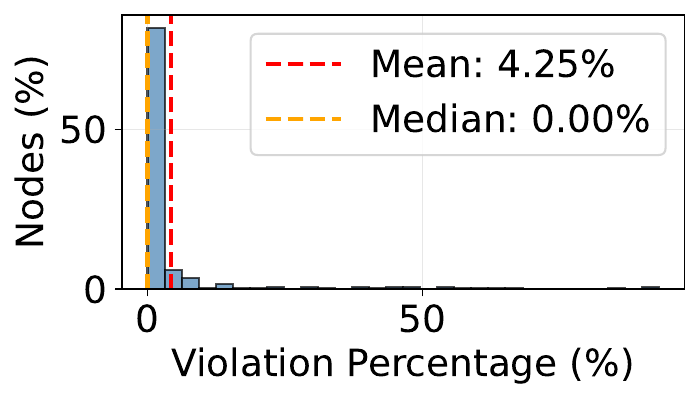}
  \includegraphics[width=0.48\linewidth]{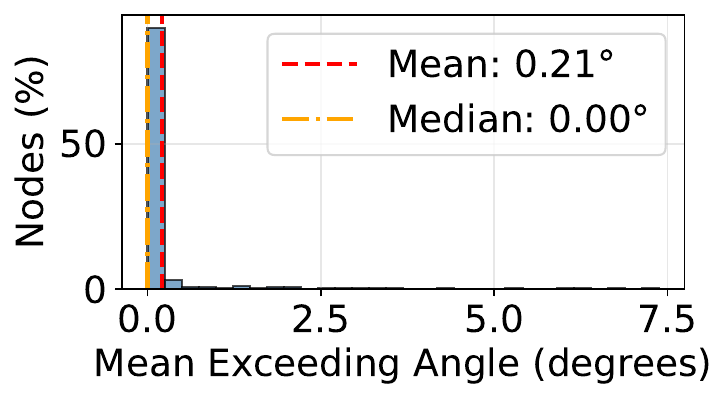}
  \vspace{-0.3cm}
  \caption{\textbf{CLIP image embeddings for ImageNet are tightly clustered.} (Left): For a fixed node $i$, we show the proportion of non-descendants of $i$ whose embeddings intersect the cone of $i$. (Right): For the same node $i$, we show the mean angle violation, with non-violating angles counted as $0$.}
  \label{fig:well_clustered_branches}
  \vspace{-0.2cm}
\end{figure}

\begin{figure}[t]
  \centering
  \includegraphics[width=\linewidth]{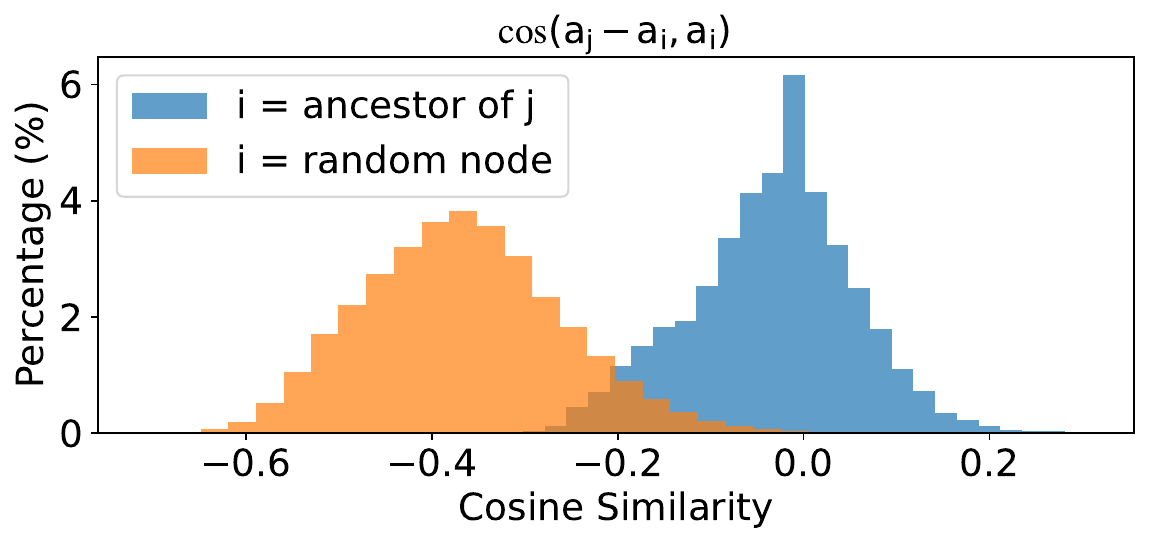}
  \vspace{-0.7cm}
  \caption{\textbf{Observed hierarchical orthogonality in CLIP image embeddings for ImageNet.} The cosine similarity between child-parent difference vectors and their parents is close to zero, while random pairings have significantly higher cosine similarity. This suggests that hierarchical orthogonality \citep{park_geometry_2025} holds even in pre-trained models.} 
  \label{fig:orthogonality_test_imagenet}
  \vspace{-0.3cm}
\end{figure}

\section{Experiments}\label{sec:experiments}
In this section, we show that \ours{} improves hierarchical concept recovery on both synthetic (\Cref{sec:synthetic_experiments}) and real-world datasets (\Cref{sec:real_experiments}). Across settings, the main empirical pattern is consistent: HB-OMP recovers the correct support more accurately than OMP, and these gains are especially pronounced in few-shot regimes where the hierarchy provides a strong inductive bias.

\begin{figure*}[t]
  \centering
  \includegraphics[width=\textwidth,clip,trim={0 1cm 0 0}]{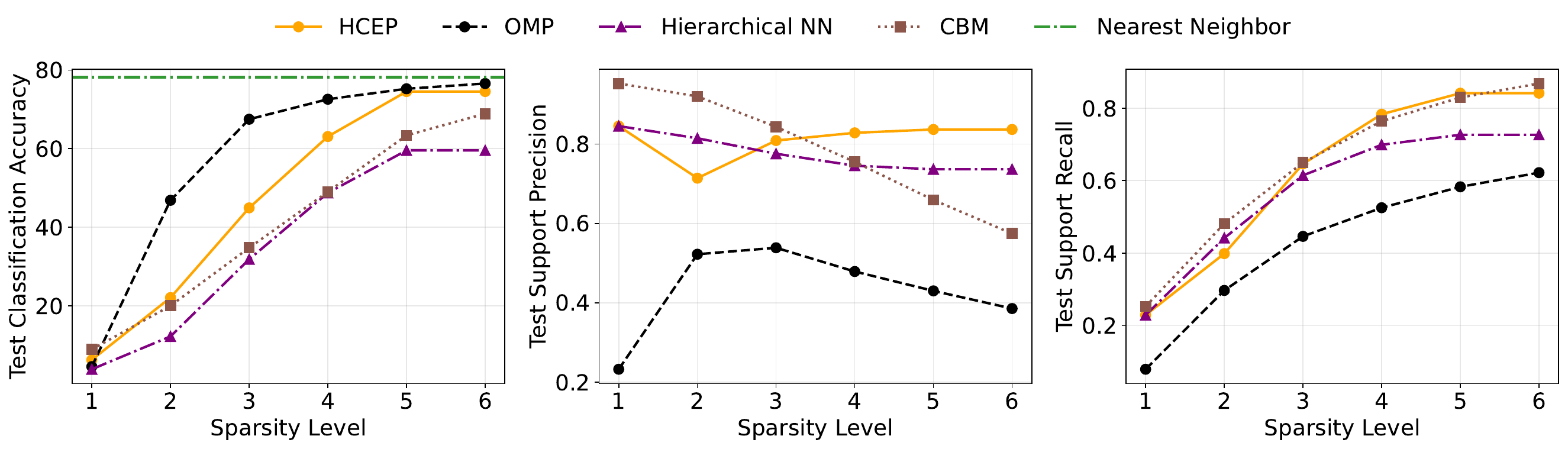}
  \includegraphics[width=\textwidth,clip,trim={0 0 0 1.2cm}]{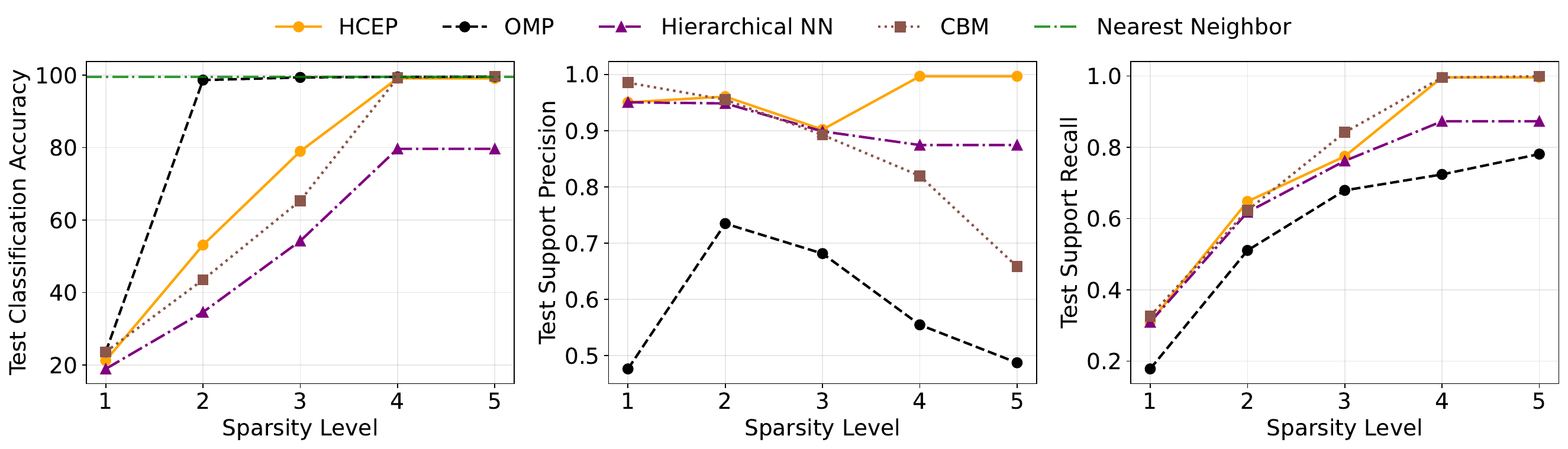}
  \vspace{-0.7cm}
  \caption{Interpretable image classification on CIFAR-100 (top) and ImageNette (bottom). \ours{} achieves state-of-the-art support precision and recall while maintaining comparable classification accuracy at a sparsity level matching the depth of the hierarchy ($5$ for CIFAR-100, $4$ for ImageNette).}
  \label{fig:cifar100_results}
  \vspace{-0.4cm}
\end{figure*}

\begin{figure*}[t]
  \centering
  \includegraphics[width=\textwidth]{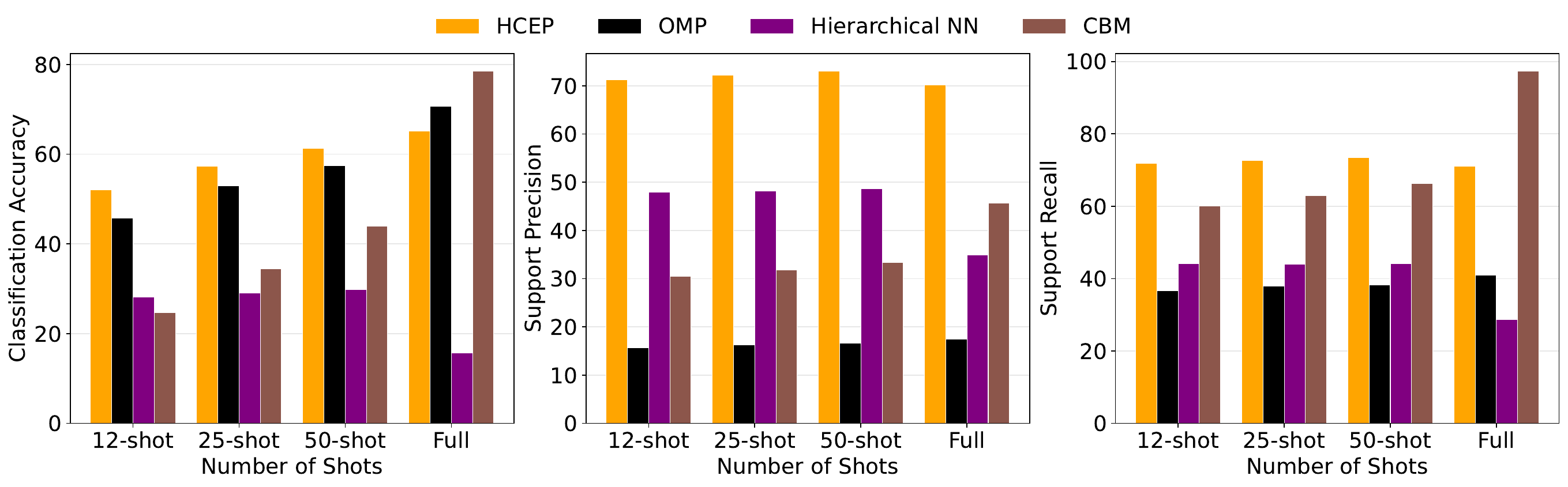}
  \vspace{-0.7cm}
  \caption{As we reduce the number of images per class in ImageNet, \ours{} consistently improves test classification accuracy, support precision, and support recall over all baselines. Results are shown at sparsity level 14.}
  \label{fig:imagenet_low_data_results}
  \vspace{-0.4cm}
\end{figure*}

\subsection{Synthetic Experiments}\label{sec:synthetic_experiments}

We first compare the performance of HB-OMP (\Cref{alg:beam-path-hierarchical-omp}) with OMP (\Cref{alg:omp}) on synthetic data generated from the Hierarchical Concept Embeddings model in \Cref{sec:generative_model}. We evaluate the reconstruction error and the recovery of the ground-truth sparse support (i.e., the path from the root to the leaf node) under varying noise levels and hierarchy depths.

We choose a branching ratio of $b=3$, a hierarchy depth of $L=7$, and a dimension $d=50$. Note that this dimension satisfies the depth-banch-dimension condition in \Cref{prop:depth_dimension} since $d=50 \ge 7 + 3 = 10$. This gives us $2,187$ leaf synsets and $3,280$ atoms in the dictionary. We generate 5 samples per leaf for a total of $10,935$ samples. Further details on how the synthetic data is generated can be found in \Cref{app:step_by_step_construction}, and the hyperparameters can be found in \Cref{app:synthetic_experiment_details}.

As it can be observed in \Cref{fig:precision_recall_index}, HB-OMP  consistently outperforms OMP in both precision and recall. This demonstrates the effectiveness of incorporating hierarchical structure into sparse coding for improved concept recovery.

\subsection{Real-World Experiments}\label{sec:real_experiments}

In this section, we evaluate \ours{} on real-world image classification tasks. We first describe how we construct the hierarchy and dictionary, then present the experimental setting and results. More details are provided in \Cref{app:real_data_experiment_details}.

\myparagraph{Datasets} We evaluate on (1) ImageNette \cite{howard2019imagnette}, (2) ImageNet \cite{deng2009imagenet}, and (3) CIFAR-100 \cite{krizhevsky2009learning}. For ImageNet-based datasets, we use the WordNet hierarchy, which has $L=14$ levels with branching ratios up to $b=25$, demonstrating that \ours{} can handle complex, large-scale hierarchies. For CIFAR-100, we use taxonomy induction methods \cite{zeng2024chain} to construct a hierarchy over the classes.

\myparagraph{Hierarchy and Dictionary Construction} We obtain the embedding of a class, which is a leaf node, as the average of the CLIP image embeddings for that class. We obtain the embedding of non-leaf synsets as the average of their children's embeddings. We construct the hierarchical dictionary as in \eqref{eq:full_dictionary} and keep it fixed throughout the experiments.

\myparagraph{Baselines} 
(1) OMP \cite{pati1993orthogonal,chattopadhyay_information_2023} with the full dictionary; (2) CBMs \cite{koh2020concept} which use supervised concept annotations to learn a concept predictor; (3) a Nearest Neighbor (NN) classifier that uses the class synset embeddings directly, which can be viewed as a black-box classifier; and (4) Hierarchical NN (HNN), which traverses the hierarchy using nearest neighbor search at each level and serves as a hierarchical interpretable-by-design baseline.

\myparagraph{Evaluation Metrics} We evaluate the models based on (1) classification accuracy; (2) support precision and recall, which measure how well the recovered sparse support matches the ground-truth path in the hierarchy. Note that we do \emph{not} assume that each leaf has a unique root-to-leaf path: in DAG-structured hierarchies such as WordNet, a synset may have multiple parents, yielding several valid paths. \ours{} explores and selects one of these paths, and our evaluation scores the recovered support against the closest among all valid root-to-leaf paths.

\myparagraph{Classification Procedure} For each image, we compute its CLIP embedding, and each method produces its own  intermediate representation: (1) OMP and HB-OMP recover a sparse code; (2) HNN recovers a 0/1 sparse code by going down a rooted path in the hierarchy using nearest neighbor search at each level and record the nodes along the path in the sparse code; and (3) CBMs produce scores for all atoms at once by training a classifier on top of the CLIP image embeddings. The recovered intermediate representations are then fed to a linear classifier trained on the training set to predict the class labels. For NN, we directly use the class synset embeddings to classify images without any intermediate representation.

\myparagraph{Results} See the results on ImageNet (\Cref{fig:imagenet_low_data_results}) and CIFAR-100/ImageNette (\Cref{fig:cifar100_results}). Across all three datasets, HB-OMP achieves higher support precision and recall than the other interpretable baselines, indicating more accurate recovery of the relevant hierarchical concepts. These gains are more pronounced in the few-shot ImageNet setting, where HB-OMP outperforms all interpretable baselines in both classification accuracy and support precision/recall. 

This pattern is consistent with our hypothesis that sparse coding methods are especially helpful when concepts must be estimated from few labeled examples: OMP and HB-OMP only require estimating synset embeddings by averaging image embeddings, whereas CBMs must learn a concept classifier from labels.
This few-shot scenario is common in real-world applications where labeled data is scarce but a pre-trained vision-language model can provide rich embeddings.
Moreover, restricting the search to hierarchically valid paths further reduces spurious atom selections, so HB-OMP performs better than OMP in few-shot settings. 

We also verified that \ours{} generalizes across vision-language models: replacing CLIP with SigLIP \citep{zhai2023sigmoid} on ImageNet yields similar improvements in interpretability metrics (see \Cref{fig:imagenet_siglip_results}).

\myparagraph{Well-Clustered Synsets and Hierarchical Orthogonality in Real-World Datasets} 
We empirically verify the geometric conditions described in \Cref{sec:generative_model} on real-world datasets. We focus on ImageNet \citep{deng2009imagenet} and use CLIP image embeddings as the representation space for the synsets and concepts.
For the well-clustered synset condition (\Cref{prop:well_clustered_hierarchy}), \Cref{fig:well_clustered_branches} shows that most branches are tightly clustered and well separated from other branches.
To test the validity of the hierarchical orthogonality condition in \Cref{sec:hierarchical_independence}, we measure the cosine similarity between child-parent difference vectors and their parent embeddings. As shown in \Cref{fig:orthogonality_test_imagenet} for ImageNet (see \Cref{fig:orthogonality_test_cifar100} for CIFAR-100), the child-parent difference vectors are close to orthogonal to the parent vectors, whereas random pairings are not.

\myparagraph{Runtime Analysis} \Cref{fig:runtime_vs_metrics} reports a runtime analysis using the same setting as \Cref{fig:imagenet_low_data_results}. As \ours{}'s beam width increases from $1$ to $32$, support precision improves with only a modest runtime overhead, enabled by parallelizing the beam-search hypotheses on GPU.

\begin{figure}[t]
\centering
\includegraphics[width=\linewidth]{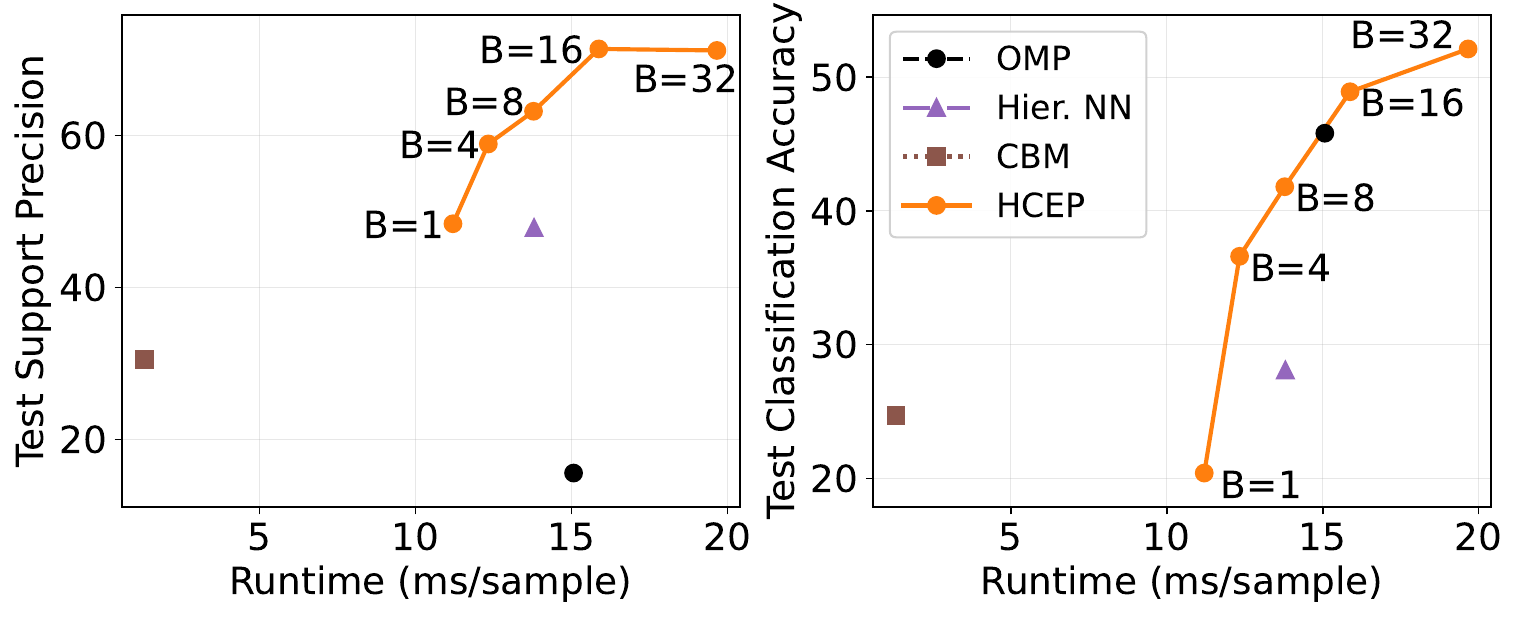}
\vspace{-0.7cm}
\caption{Runtime, support precision and classification accuracy on ImageNet. All methods use $12$-shot and sparsity level $14$.}
\vspace{-0.5cm}
\label{fig:runtime_vs_metrics}
\end{figure}

\section{Conclusion}\label{sec:conclusion}

We introduced \ours{}, a geometric framework for hierarchical concept embeddings and a pursuit algorithm that respects the structure of synset hierarchies. We analyzed identifiability requirements through well-clustered cones, hierarchical orthogonality, and simplex structure, along with support-recovery guarantees for Hierarchical OMP. Empirically, the resulting codes deliver better support precision and recall than interpretable baselines across synthetic and real-world benchmarks.
Our findings highlight the promise of structured sparse coding as a scalable and flexible framework for interpretable machine learning.

\section*{Acknowledgements}
The authors thank Ryan Chan, Ryan Pilgrim, Uday Kiran Tadipatri, Buyun Liang, Kyle Poe, and anonymous reviewers for their helpful comments and suggestions. This work was funded by Simons Foundation grant 814201, NSF grant 2031985, and University of Pennsylvania startup funds.

{
    \small
    \bibliographystyle{ieeenat_fullname}
    \bibliography{main,nghia}
}

\clearpage
\appendix
\tableofcontents

\section{Extended Related Work}\label{app:related_work}
\myparagraph{Interpretable-by-design models} Interpretable-by-design models aim to provide explanations for their predictions by using human-interpretable concepts as intermediate representations. Early works explored attribute-based classification for face verification \cite{kumar2009attribute} and learning to detect unseen object classes through attribute transfer \cite{lampert2009learning}. Subsequent works include Concept Activation Vectors \cite{kim2018interpretability}, which use linear classifiers to identify directions in the embedding space corresponding to specific concepts. Concept Bottleneck Models \cite{koh2020concept} extend this idea by training a model to predict concepts before predicting the final output. In an adjacent line of work, Information Pursuit \cite{geman1996active} is used as a criterion to choose the most relevant concepts \cite{chattopadhyay2023variational,chattopadhyay2024bootstrapping,kolek2025learning}. More recent works have explored leveraging pre-trained embeddings and sparse coding for identifying specific concept directions \cite{chattopadhyay_information_2023,bhalla_interpreting_2024,gandelsman_interpreting_2024-1}. Our work builds upon these foundations by introducing a concept embedding framework that captures the hierarchical relationships among synsets in interpretable image classification.

\myparagraph{Sparse Recovery}
Sparse recovery aims to recover a sparse signal from a set of observations, often using techniques such as Orthogonal Matching Pursuit (OMP) \cite{pati1993orthogonal} and Basis Pursuit \cite{chen2001atomic}. 
Sparse coding has been widely used in image processing \cite{mairal2008learning,mairal2010online}, signal processing, and machine learning.
Although there have been works on hierarchical sparse coding \cite{jost_tree-based_2006, jenatton2011proximal,la_tree-based_2006}, they do not consider the hierarchical structure of concepts in the context of interpretable models or deep representation learning.
More recently, sparse autoencoders (SAEs) \cite{ng2011sparse,cunningham2023sparse,bricken2023monosemanticity} have been used in vision-language models \citep{zaigrajew_interpreting_2025,olson_analyzing_2025,costa_flat_2025} to interpret the structure of concepts beyond sparsity. Unlike SAE-based approaches, our work focuses on structured sparse coding with an explicitly hierarchical dictionary derived from semantic synset relationships.

\section{Preliminaries}
\subsection{Interpretable-by-design vs. Chain-of-Thought}\label{sec:interp_design_vs_cot}
Chain-of-Thought (CoT) prompting \citep{wei2022chain} is a technique that enables large language models to generate step-by-step reasoning traces before producing a final answer, often improving performance on complex reasoning tasks.
However, using CoT as an interpretability method does not guarantee faithfulness 
\footnote{An explanation is \emph{faithful} if it accurately reflects the true computation process to the final prediction \citep{jacovi_goldberg_2020_towards}.}
because the final prediction is conditioned on both the input and the generated chain of thoughts. Several recent audits further question the faithfulness of CoT explanations \citep{barezChainofThoughtNotExplainability,lanhamMeasuringFaithfulnessChainofThought2023,turpinLanguageModelsDont2023}. 
Therefore, the development of CoT does not make interpretable-by-design models obsolete. In fact, there is a trend toward making foundation models (e.g., LLMs or diffusion models) interpretable-by-design by enforcing an interpretable concept bottleneck in the latent space \citep{sunConceptBottleneckLLM2025,ismailConceptBottleneckGenerative2024}.

\subsection{Canonical regular simplex}\label{sec:canonical_simplex}
Define
\begin{equation}\label{eq:simplex_definition}
  \tilde{\bs}_{j}
  \;=\;
  \be_{j}
  \;-\;
  \frac{1}{b}\,\mathbf 1,
  \qquad
  j = 1,\ldots,b,
\end{equation}
where $\{\be_{j}\}_{j=1}^{b}$ are the standard basis vectors of $\R^{b}$ and
$\mathbf 1\in\R^{b}$ is the all-ones vector.
These centred vertices satisfy
$\sum_{j=1}^{b}\tilde{\bs}_{j}=\mathbf0$ and
\(
  \tilde{\bs}_{j}^{\!\top}\tilde{\bs}_{k}
  =
  \begin{cases}
    1, & j=k,\\[4pt]
    -\dfrac{1}{b-1}, & j\neq k,
  \end{cases}
\)
i.e.\ they form a regular $(b-1)$-simplex of unit edge length in~$\R^{b-1}$.

\section{Limitations}
While our framework demonstrates clear advantages in concept recovery, it also has several limitations:

\myparagraph{Embedding Constraints} Our theoretical analysis (\Cref{prop:depth_dimension}) establishes that embedding a hierarchy with leaf depth $L$ and branching ratio $b$ requires ambient dimension $d \ge L + b - 1$. For deep hierarchies (large $L$) or highly branching structures (large $b$), this constraint becomes restrictive. Real-world embeddings from models such as CLIP typically have fixed dimensions (e.g., $d=768$), which limits the depth and complexity of hierarchies that can be faithfully represented. Moreover, as hierarchies deepen, the half-angles of the cones containing each subtree (\Cref{prop:angle_schedule}) must decrease geometrically. As mentioned in \Cref{sec:well_clustered_synsets}, this limitation may necessitate exploring alternative geometries, such as hyperbolic spaces \citep{nickel2017poincare,desai_hyperbolic_2024}, for more faithful hierarchical representations. At the same time, by using pretrained embeddings, \ours{} provides accurate and interpretable classification without the additional compute required to finetune large models, and pretrained models are available and improving across multiple domains, making \ours{} easy to extend. That said, hierarchy-aware finetuning could further improve the geometric conditions in \Cref{sec:generative_model} and thus boost \ours{}'s performance; we leave this direction to future work.

\myparagraph{Hierarchy Quality Dependence} The performance of Hierarchical OMP critically depends on the quality of the predefined hierarchy. For ImageNet-based datasets, we leverage the well-curated WordNet hierarchy, which provides semantically meaningful relationships. However, for CIFAR-100, we rely on taxonomy induction methods~\citep{zeng2024chain}, which may produce hierarchies with inconsistencies or unclear relationships. That said, high-quality hierarchies already exist in many domains beyond common objects, e.g., RadLex \citep{langlotz2006radlex} for radiology, DERM12345 \citep{yilmaz2024derm12345} for skin lesions, and iNaturalist \citep{van2018inaturalist} for species classification. Moreover, the quality of LLM-based taxonomy induction is steadily improving, as demonstrated by our CIFAR-100 experiments (\Cref{sec:real_experiments}). When the hierarchy is noisy, \ours{} may produce incorrect explanation paths. A promising direction for future work is to interpolate between hierarchical and non-hierarchical solutions via $\ell_1$ minimization with a hierarchy regularizer \citep{jenatton2011proximal}, thereby allowing the method to degrade gracefully when hierarchy quality is uncertain.

\myparagraph{Computational Complexity} Hierarchical OMP with beam search (\Cref{alg:beam-path-hierarchical-omp}) has complexity $O(TBb|\mathcal{D}_{\text{active}}|)$, where $T$ is the number of iterations, $B$ is the beam width, $b$ is the branching ratio, and $|\mathcal{D}_{\text{active}}|$ is the size of the active dictionary at each level. In contrast, OMP has complexity $O(T|\mathcal{D}|)$, where $|\mathcal{D}|$ is the size of the entire dictionary. For large branching ratios or deep hierarchies, this may become computationally expensive. 
While we demonstrate better concept recovery accuracy over standard OMP, the computational cost remains a practical consideration for deployment at scale. In practice, the beam search hypotheses can be parallelized on GPU, which significantly reduces wall-clock time overhead (see \Cref{fig:runtime_vs_metrics}).

\myparagraph{Outlier Handling} \ours{} assumes that the input belongs to a leaf class in the provided hierarchy. For outliers within a known class (e.g., a cat with three legs), classical results on the robustness of sparse coding to corruption \citep{wright2009robust} suggest that the outlier will still be closest to the correct synset path, while a larger sparse reconstruction error can flag such cases. If a novel class is entirely absent from the hierarchy, taxonomy induction can be used to append the new class before applying \ours{}.

\section{Proofs}\label{app:proofs}

\subsection{Proof of \Cref{prop:well_clustered_hierarchy}}\label{app:proof-well-clustered-hierarchy}

\myparagraph{Statement}\quad 
If subtree containment (Eq.~\eqref{eq:subtree_well_clustered}) and sibling-cone disjointness (Eq.~\eqref{eq:cone_disjoint}) hold, 
then the subtrees rooted at sibling nodes do not overlap.

\begin{proof}
Suppose $j$ and $j'$ are distinct children of a parent node $i$. We show that their subtrees are disjoint.

Let $k \in \desc(j)$. By subtree containment in Eq.~\eqref{eq:subtree_well_clustered},
\begin{equation}
  \angle(\ba^{(j)}, \ba^{(k)}) \leq \theta_{\lev(j)}.
\end{equation}

Assume for contradiction that $k \in \desc(j')$ as well. Applying Eq.~\eqref{eq:subtree_well_clustered} to $j'$ gives
\begin{equation}
  \angle(\ba^{(j')}, \ba^{(k)}) \leq \theta_{\lev(j')}.
\end{equation}
Now consider the spherical triangle formed by the unit vectors
$\ba^{(j)}/\|\ba^{(j)}\|$, $\ba^{(k)}/\|\ba^{(k)}\|$, and $\ba^{(j')}/\|\ba^{(j')}\|$.
The triangle inequality yields
\begin{align}
  \angle(\ba^{(j)}, \ba^{(j')}) 
  &\leq \angle(\ba^{(j)}, \ba^{(k)}) + \angle(\ba^{(k)}, \ba^{(j')})
  \notag \\ 
  &\leq \theta_{\lev(j)} + \theta_{\lev(j')},
\end{align}
which contradicts sibling-cone disjointness in Eq.~\eqref{eq:cone_disjoint}. Therefore $k$ cannot belong to both $\desc(j)$ and $\desc(j')$. Since $j$ and $j'$ were arbitrary siblings, the subtrees rooted at sibling nodes are disjoint.
\end{proof}

\subsection{Proof of \Cref{prop:angle_schedule}}\label{app:proof-angle-schedule}

\myparagraph{Statement}\quad 
If the half-angles satisfy $\theta_{l+1} \le \min\{r,1/b\}\,\theta_l$ with $r\in(0,1/2)$, then there exists a placement of the nodes such that \textit{subtree containment} (Eq.~\eqref{eq:subtree_well_clustered}) and \textit{sibling-cone disjointness} (Eq.~\eqref{eq:cone_disjoint}) hold.

\begin{proof}
We verify the two requirements separately.

\textbf{Step 1: Subtree containment.}
A sufficient condition for Eq.~\eqref{eq:subtree_well_clustered} is that the cumulative half-angles of all lower levels do not exceed the half-angle budget at level $l$, namely
\begin{equation}\label{eq:bounded_cumulative_half_angles}
  \sum_{k=l+1}^L \theta_k \leq \theta_l,
  \quad \forall i \in \{1, \ldots, N_{L}\}, \; l= \lev(i).
\end{equation}
Assume first that the half-angles satisfy the geometric decrease
\begin{equation}\label{eq:geometric_decrease}
  \theta_{l+1} \le r\,\theta_l.
\end{equation}
Then, for any level $l$,
\begin{align}\label{eq:sum_half_angles}
  \sum_{k=l+1}^L \theta_k
  &\le \sum_{k=1}^{L-l} \theta_l\, r^k \text{(by Eq.~\eqref{eq:geometric_decrease})}
  \\&= \theta_l\,\sum_{k=1}^{L-l}  r^k
  \\&= \theta_l\, r\, \sum_{k=0}^{L-l-1} r^k
  \\&\le \theta_l\, r \, \frac{1-r^{L-l-1}}{1-r} \text{(sum of a geometric series)}.
\end{align}
Taking $L\to\infty$ yields
\begin{align*}\label{eq:infinite_sum_half_angles}
  	\theta_l\, r \frac{1-r^{L-l-1}}{1-r} \to \theta_l\, \frac{r}{1-r} \le \theta_l, \qquad r < \frac{1}{2},
\end{align*}
so Eq.~\eqref{eq:bounded_cumulative_half_angles} holds. This proves subtree containment.

\textbf{Step 2: Sibling-cone disjointness.}
Next, we show that Eq.~\eqref{eq:cone_disjoint} follows from the simpler sufficient condition
\begin{equation}\label{eq:cone_packing_sufficient}
  \theta_{l+1} \le \frac{1}{b}\,\theta_l.
\end{equation}
Consider any 2D plane containing the parent cone axis. In that plane, a cone of half-angle $\alpha$ appears as a planar angle of magnitude $2\alpha$. Therefore, one can place $b$ child cone axes inside the parent cone without overlap whenever $b$ child angles of size $2\theta_{l+1}$ fit inside the parent angle $2\theta_l$, i.e., whenever $b\theta_{l+1}\le\theta_l$. This is exactly Eq.~\eqref{eq:cone_packing_sufficient}.

Since the assumption of the proposition is $\theta_{l+1} \le \min\{r,1/b\}\theta_l$, both Step 1 and Step 2 hold simultaneously. Hence there exists a placement satisfying subtree containment and sibling-cone disjointness.
\end{proof}

\subsection{Proof of \Cref{prop:depth_dimension}}\label{app:proof-depth-dimension}

\noindent\textbf{Statement.}\quad Under the hierarchical orthogonality constraints in Eq.~\eqref{eq:multi_ortho} and the regular-simplex difference condition in Eq.~\eqref{eq:multi_simplex_ip} at every internal node up to depth $L-1$ of a hierarchy whose leaves are at depth $L$ in ambient space $\R^d$, it is necessary that the ambient dimension satisfies the depth--dimension condition $d \ge L + b - 1$.

\begin{proof}
Fix a level $l\ge 0$ and consider the ancestor path $\bA_l=\{\ba^{(\pi_0)},\ldots,\ba^{(\pi_l)}\}$. The hierarchical orthogonality constraints in Eq.~\eqref{eq:multi_ortho} define the affine feasible set for child candidates:
\begin{align*}
  \cV_l = \{\bx\in\R^d : \bA_l^{\top}\bx = \bh_l\},
\end{align*}
which is Eq.~\eqref{eq:Vl_def}. When the ancestor vectors are linearly independent, as is generically the case because each level introduces a new non-collinear direction, we have
\begin{align*}
  \dim \cV_l = d - (l+1).
\end{align*}

Now apply the regular-simplex condition in Eq.~\eqref{eq:multi_simplex_ip}. It requires placing $b$ child points whose differences relative to a feasible origin in $\cV_l$ form a regular $(b{-}1)$-simplex. Since such a simplex has affine hull of dimension $b{-}1$, it can be embedded in $\cV_l$ only if
\begin{align*}
  \dim \cV_l \;\ge\; b-1.
\end{align*}
Combining the two displays yields, for every level $l$,
\(
  d - (l+1) \ge b - 1 \;\Longleftrightarrow\; d \ge l + b.
\)
Requiring this inequality to hold at the deepest internal level $l=L-1$ gives $d \ge (L-1) + b = L + b - 1$, as claimed.
\end{proof}

\subsection{Intermediate results for \Cref{prop:hier_omp_erc}}

\begin{lemma}[Column normalization equivalence]
\label{lem:normalization}
Let $\bD=[\bd_1,\dots,\bd_k]\in\R^{d\times k}$ with arbitrary nonzero column norms ($\|\bd_j\|_2>0$ for all $j\in[k]$), and define the diagonal matrix
$\bW:=\mathrm{diag}(\|\bd_1\|_2,\dots,\|\bd_k\|_2)$ and the column-normalized dictionary
$\widehat{\bD}:=\bD \, \bW^{-1}$.
For any $s$-sparse $\bz\in\R^{k}$ with support $S$, set $\widehat{\bz}:=\bW\bz$.
Then $\bx=\bD\bz=\widehat{\bD}\,\widehat{\bz}$ and $\mathrm{supp}(\widehat{\bz})=S$.
Moreover, OMP run on $\bD$ with the selection rule
\begin{equation}
  j^\star \,\in\, \arg\max_{j} \; \frac{\big|\langle \br,\bd_j\rangle\big|}{\|\bd_j\|_2}
  \,=\, \arg\max_{j} \; \frac{\big|\langle \br,\bd_j\rangle\big|}{\|\bd_j\|_2\,\|\br\|_2}
\end{equation}
is identical (same index picked at every iteration) to OMP run on $\widehat{\bD}$ with the usual (unnormalized) correlation rule.
Equivalently, this selects the atom with the highest absolute cosine similarity to the residual.
\end{lemma}

\begin{proof}
The representation identity is immediate:
\begin{equation}
\bx=\bD\bz=\bD\,\bW^{-1}\bW\bz=\widehat{\bD}\,\widehat{\bz},
\end{equation}
and clearly $\mathrm{supp}(\widehat{\bz})=S$ because $\bW$ is diagonal with positive entries.

For the selection rule, note that $\widehat{\bd}_j=\bd_j/\|\bd_j\|_2$, so
\begin{equation}
\langle \br,\widehat{\bd}_j\rangle
=
\frac{\langle \br,\bd_j\rangle}{\|\bd_j\|_2}.
\end{equation}
Thus maximizing correlation with $\widehat{\bd}_j$ is exactly the same as maximizing normalized correlation, equivalently absolute cosine similarity, with $\bd_j$.

Finally, diagonal rescaling of the columns in $\bD_S$ does not change $\mathrm{span}(\bD_S)$, so the orthogonal projector onto this span is invariant under the rescaling $\bD_S \mapsto \widehat{\bD}_S$. 
Therefore, once the same index is selected, both versions of OMP produce the same least-squares fit and hence the same residual at every step. By induction over the iterations, the selected indices are identical throughout the run.
\end{proof}

\begin{definition}[Exact Recovery Coefficient (ERC) on normalized dictionary]\footnote{We use Exact Recovery Coefficient (ERC) for the quantity whose threshold at $1$ is the classical exact recovery condition in \citet{Tropp2004GreedIsGood}.}
\label{def:erc-hat}
For a support $S$ with $\widehat{\bD}_S$ full column rank, define
\begin{align}
  \mathrm{ERC}(\widehat{\bD};S)
  &:= \max_{j\in S^c} \bigl\|\,\widehat{\bD}_S^{\dagger}\,\widehat{\bd}_j\,\bigr\|_1,\\
  \mathrm{ERC}(\widehat{\bD};S\,|\,T)
  &:= \max_{j\in T\setminus S} \bigl\|\,\widehat{\bD}_S^{\dagger}\,\widehat{\bd}_j\,\bigr\|_1,
\end{align}
for any $T\supseteq S$, where $\|\cdot\|_1$ denotes the vector $\ell_1$ norm.
Here $S^c:=[k]\setminus S$, so $T\setminus S = T\cap S^c \subseteq S^c$; thus $\mathrm{ERC}(\widehat{\bD};S\,|\,T)$ is the same maximum as $\mathrm{ERC}(\widehat{\bD};S)$, but over the smaller index set $T\setminus S$.
\end{definition}

\begin{lemma}[Monotone ERC improvement under subtree restriction]
\label{lem:erc_monotone}
Let $\bD\in\R^{d\times k}$ have arbitrary nonzero column norms and let $\bz$ be $s$-sparse with support $S$.
Let $T_0 \supset T_1 \supset \cdots \supset T_L$ be a nested sequence with $T_0=[k]$ and $S\subseteq T_\ell$ for all $\ell=0,\dots,L$.
Assume $\widehat{\bD}_S$ has full column rank.
Then the ERC decreases monotonically along the restriction:
\begin{align}
  \mathrm{ERC}(\widehat{\bD};S\,|\,T_L)
  &\le \mathrm{ERC}(\widehat{\bD};S\,|\,T_{L-1})\\
  &\le \cdots \le \mathrm{ERC}(\widehat{\bD};S\,|\,T_0)\\
  &:= \mathrm{ERC}(\widehat{\bD};S).
\end{align}
\end{lemma}

\begin{proof}
By definition, $\mathrm{ERC}(\widehat{\bD};S\,|\,T)=\max_{j\in T\setminus S} \|\widehat{\bD}_S^\dagger \widehat{\bd}_j\|_1$. Since $\widehat{\bD}_S^\dagger$ is fixed, shrinking $T$ only restricts the index set over which this maximum is taken, so the value cannot increase.
\end{proof}

\begin{lemma}[ERC threshold implies restricted OMP success]
\label{lem:erc_success}
Under the assumptions of \Cref{lem:erc_monotone}, if $\mathrm{ERC}(\widehat{\bD};S\,|\,T_L)<1$, then OMP run on $\bD$ with the normalized selection rule
\begin{equation}
  j^* \in \arg\max_j \frac{|\langle \br,\bd_j\rangle|}{\|\bd_j\|_2}
  \,=\, \arg\max_j \frac{|\langle \br,\bd_j\rangle|}{\|\bd_j\|_2\,\|\br\|_2},
\end{equation}
restricted to $T_L$, recovers $S$ in $s$ steps. Equivalently, OMP on $\widehat{\bD}_{T_L}$ with the standard rule succeeds in $s$ iterations.
\end{lemma}

\begin{proof}
\Cref{lem:normalization} shows that the normalized-selection rule on $\bD$ matches standard OMP on $\widehat{\bD}$. The classical noiseless ERC theorem of \citet{Tropp2004GreedIsGood} applied to the restricted dictionary $\widehat{\bD}_{T_L}$ then yields exact support recovery in $s$ iterations whenever $\max_{j\in T_L\setminus S} \|\widehat{\bD}_S^\dagger \widehat{\bd}_j\|_1 < 1$, equivalently whenever $\mathrm{ERC}(\widehat{\bD};S\,|\,T_L)<1$.
\end{proof}

\subsection{Proof of \Cref{prop:hier_omp_erc}}\label{app:proof-hier-omp-erc}
\noindent\textbf{Statement.}\quad
There exist instances with $\mathrm{ERC}(\widehat{\bD};S)\ge 1$ yet $\mathrm{ERC}(\widehat{\bD};S\,|\,T_L)<1$ for some nested $T_0 \supset T_1 \supset \cdots \supset T_L$ satisfying the right-subtree assumption $S\subseteq T_\ell$.
Consequently, hierarchical OMP yields a strictly larger ERC-certified success region than global OMP on the full dictionary.

\begin{proof}
Choose any instance for which the maximizer of $\mathrm{ERC}(\widehat{\bD};S)$ lies outside $T_L$. Removing that maximizer from the admissible index set strictly decreases the maximum, so
\[
\mathrm{ERC}(\widehat{\bD};S\,|\,T_L)
<
\mathrm{ERC}(\widehat{\bD};S).
\]
Hence it is possible to have
\[
\mathrm{ERC}(\widehat{\bD};S)\ge 1
\qquad\text{but}\qquad
\mathrm{ERC}(\widehat{\bD};S\,|\,T_L)<1.
\]
Whenever this occurs, \Cref{lem:erc_success} certifies exact recovery for Hierarchical OMP on the restricted set $T_L$, while the standard ERC guarantee for global OMP on the full dictionary does not apply.
\end{proof}

\section{Step-by-step construction of a Hierarchical Concept Embedding}\label{app:step_by_step_construction}

Assume we are at depth $l\!>\!0$ of the hierarchy. The path from the root to the \emph{current parent} $\pi_l$ (with embedding $\ba^{(\pi_l)}\in\R^{d}$) consists of the $l\!+\!1$ ancestor vectors
\begin{align}
  \bA_l = \{\ba^{(\pi_0)},\ba^{(\pi_1)},\dots,\ba^{(\pi_l)}\},
  \qquad \pi_0 < \pi_1 < \dots < \pi_l.
\end{align}
We seek embeddings $\{\ba^{(j)}\}_{j=1}^{b}\subset\R^{d}$ for its $b$ children that satisfy the following conditions:

\begin{enumerate}[label=(\roman*)]
  \item \textbf{Hierarchical Orthogonality}:
        \begin{align}\label{eq:multi_ortho}
          (\ba^{(j)} - \ba^{(\pi_k)})^{\!\top}\ba^{(\pi_k)} = 0,\quad
          &k = 0,\dots,l,
          \notag \\ &j = 1,\dots,b.
        \end{align}
        By induction, the current parent node $\pi_l$ already satisfies these orthogonality constraints with respect to its ancestors.
  \item \textbf{Regular $(b\!-\!1)$-simplex structure}:
        \begin{equation}\label{eq:multi_simplex_ip}
          (\ba^{(j)}-\bg_l)^{\!\top}(\ba^{(k)}-\bg_l)
          =
          \begin{cases}
            \lambda_l^{2}, & j = k,\\[4pt]
            -\dfrac{\lambda_l^{2}}{b-1}, & j \neq k,
          \end{cases}
        \end{equation}
        where $\bg_l$ is any point satisfying all $l\!+\!1$ equations in Eq.~\eqref{eq:multi_ortho}, and $\lambda_l$ scales the simplex so that $\angle(\ba^{(j)}, \ba^{(\pi_l)}) = \theta_l$.
  \item \textbf{Cone condition w.r.t.\ the current parent}:
        \begin{align}\label{eq:multi_cone}
          \begin{split}
          \angle(\ba^{(j)},\ba^{(\pi_l)}) = \theta_l
          \\
          \Longleftrightarrow
          \|\ba^{(j)}-\ba^{(\pi_l)}\| = \|\ba^{(\pi_l)}\|\tan\theta_l,\\
          \qquad j = 1,\dots,b.
          \end{split}
        \end{align}
        This equivalence holds because Eq.~\eqref{eq:multi_ortho} implies $\langle \ba^{(j)} - \ba^{(\pi_l)}, \ba^{(\pi_l)} \rangle = 0$.
\end{enumerate}

\subsection{Feasible Subspace Induced by Hierarchical Orthogonality}

Let
\begin{equation}
  \bA_l = \begin{bmatrix} \ba^{(\pi_0)} & \ba^{(\pi_1)} & \dots & \ba^{(\pi_l)} \end{bmatrix} \in \R^{d\times(l+1)}.
\end{equation}

The $l\!+\!1$ hyperplanes in~\eqref{eq:multi_ortho} intersect in the affine subspace
\begin{equation}\label{eq:Vl_def}
  \cV_l = \{\bx\in\R^{d}: \bA_l^{\top}\bx = \bh_l\},
\end{equation}
where $\bh_l = [\|\ba^{(\pi_0)}\|^{2},\dots,\|\ba^{(\pi_l)}\|^{2}]^{\top}$.
If the ancestor columns of $\bA_l$ are linearly independent\footnote{This is typical because every level adds a new non-collinear vector.},
then
\begin{equation}
  \dim\cV_l \;=\; d - (l+1).
\end{equation}
To be able to embed a regular $(b\!-\!1)$-simplex for all depths we therefore require the \emph{depth-dimension condition}
\begin{equation}\label{eq:depth_dim_cond}
  d \;\ge\; L + b.
\end{equation}
Equation~\eqref{eq:depth_dim_cond} quantifies the depth–dimension trade-off: one ambient degree of freedom is lost per additional ancestor constraint, while $(b\!-\!1)$ directions are always needed to accommodate the regular simplex of conditionally independent children.

We now give a constructive procedure for the children of node $\pi_l$ at level $l$.

\begin{enumerate}
  \item \textbf{Find one feasible origin.}
    Solve the linear system $\bA_l^{\!\top}\bx=\bh_l$ to obtain any particular solution $\bg_l\in\cV_l$. If we also impose the cone half-angle condition, the feasible set can be further restricted to the intersection of $\cV_l$ with the cone centered at the parent embedding $\ba^{(\pi_l)}$ and half-angle $\theta_l$ from \Cref{sec:generative_model}.
    
    A convenient choice, which preserves the full half-angle budget, is to take the current parent embedding itself as the feasible origin, i.e., $\bg_l = \ba^{(\pi_l)}$. By construction, every node is orthogonal to all of its ancestors. Hence, for each $k \in \{0,\dots,l\}$,
    \begin{align}
      (\ba^{(\pi_l)} - \ba^{(\pi_k)})^{\top}\ba^{(\pi_k)} = 0
      \\ 
      \Longrightarrow\;
      \ba^{(\pi_l)\top}\ba^{(\pi_k)} = \|\ba^{(\pi_k)}\|^{2}.
    \end{align}

  \item \textbf{Basis for the difference linear space.}
    Compute an orthonormal basis
    \begin{align}
      \bU_l &\in \R^{d\times(d-l-1)},\\
      \bA_l^{\!\top}\bU_l &= \mathbf0,\\
      \bU_l^{\!\top}\bU_l &= I_{\,d-l-1},
    \end{align}
    e.g., by taking the $d-(l+1)$ left singular vectors of~$\bA_l$ associated with the smallest singular values, denoted $\bU_{l+2:d}$.
  \item \textbf{Canonical regular simplex in $\R^{b-1}$.}
    Use the centred construction
    \(
      	\tilde{\bd}_i = \be_i - \frac1b\mathbf1,\;
      i=1,\dots,b
    \)
    (cf.\ Eq.~\eqref{eq:simplex_definition}).
  \item \textbf{Scale $\tilde{\bd}_j$ to satisfy the cone condition.}
    \begin{align}
      \lambda_l &= \|\ba^{(\pi_l)}\|\tan\theta_l,\\
      \bd_j &= \lambda_l\,\tilde{\bd}_j.
    \end{align}
  \item \textbf{Embed and translate.}
    \begin{equation}
      \ba^{(j)}
      \;=\;
      \ba^{(\pi_l)} \;+\; \bU_{l+2:d}\,\bd_j,\qquad j = 1,\dots,b.
    \end{equation}
\end{enumerate}

Choosing the scale factor
\begin{equation}\label{eq:lambda_choice}
  \lambda_l
  \;=\;
  \|\ba^{(\pi_l)}\|\tan\theta_l,
\end{equation}
forces
\(\|\ba^{(j)}-\ba^{(\pi_l)}\| = \|\ba^{(\pi_l)}\|\tan\theta_l\) for all $j$, and therefore
\(
  \angle\bigl(\ba^{(j)},\ba^{(\pi_l)}\bigr)=\theta_l
\), which is precisely the requirement
in~Eq.\,\eqref{eq:multi_cone}.

\section{Additional Experimental Results}\label{app:additional_experiments}

\begin{figure*}[ht]
  \centering
  \includegraphics[width=\textwidth]{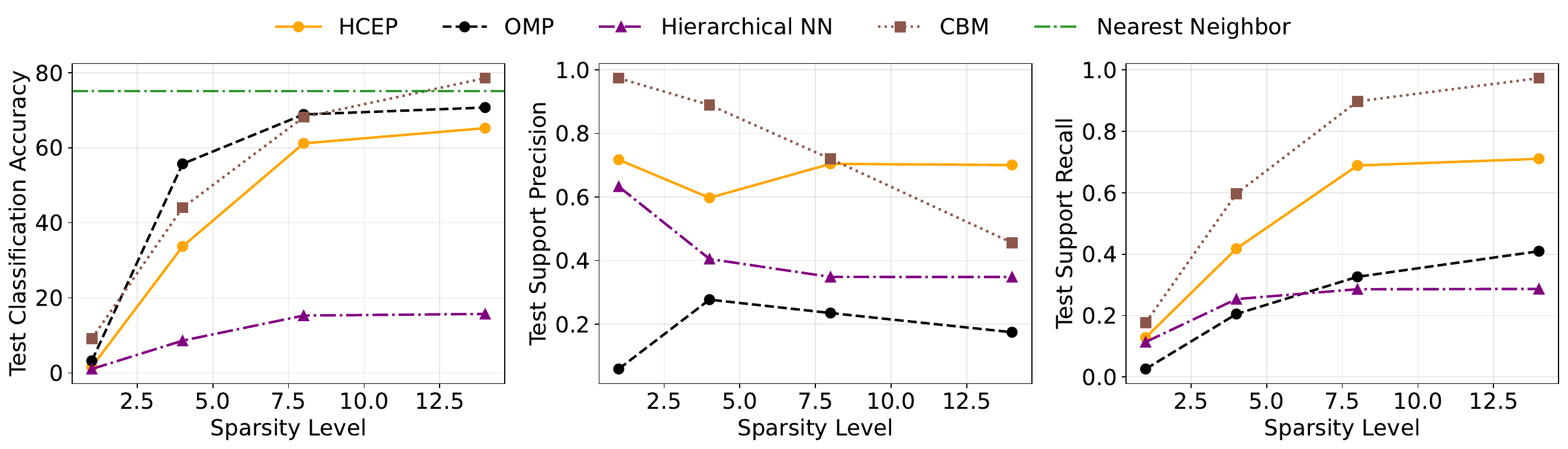}
  \caption{On ImageNet, \ours{} achieves competitive accuracy while having higher concept precision/recall than sparse concept prediction baselines (OMP, Hierarchical NN).}
  \label{fig:imagenet_results}
\end{figure*}

\begin{figure*}[ht]
  \centering
  \includegraphics[width=\textwidth]{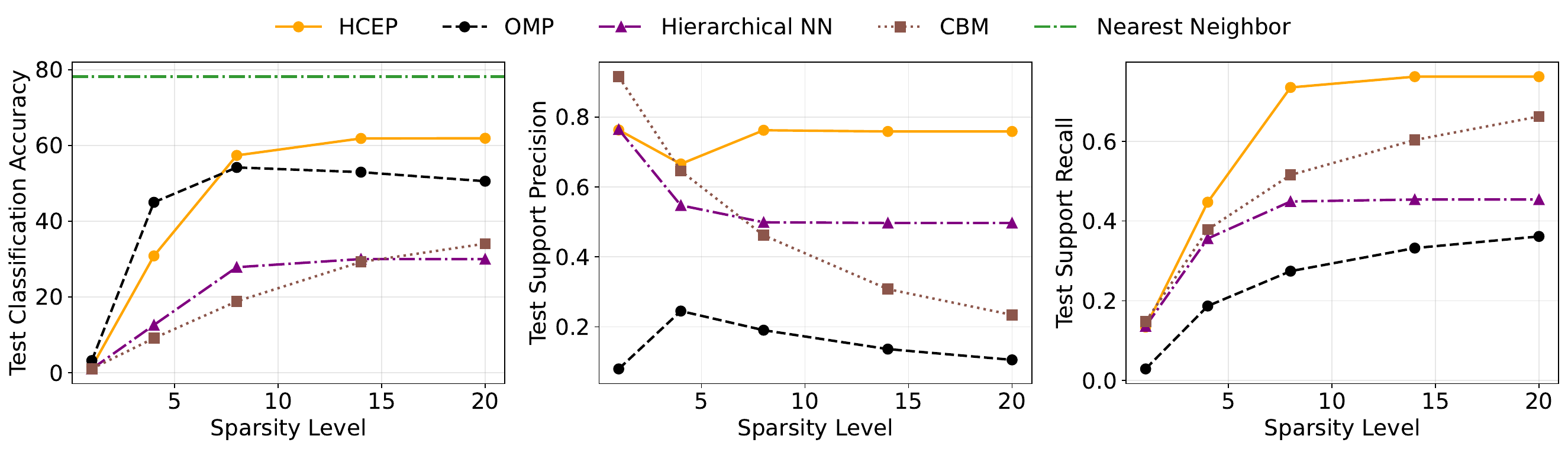}
  \caption{Interpretable image classification on ImageNet (12-shot) using SigLIP \citep{zhai2023sigmoid} embeddings. \ours{} exhibits similar improvements in support precision and recall over baselines as with CLIP (cf.\ \Cref{fig:imagenet_results}), demonstrating generality across vision-language models.}
  \label{fig:imagenet_siglip_results}
\end{figure*}

\subsection{Additional Synthetic Experiment Details}\label{app:synthetic_experiment_details}
Branching ratio $b=3$, hierarchy depth $L=7$, dimension $d=50$.
Initial cone half angle = 85 degrees. Initial vector norm = 0.8.
Geometric reduction factor = 0.4.
Total leaf nodes = 2187. Total nodes = number of atoms = 3280.
Gaussian noise for each leaf for data generation $\sigma^2=10^{-5}$.
Generate 5 samples per leaf for a total of 10,935 samples.

\subsection{Text Interpretation of Synset Differences}\label{sec:text_interpretation}
Optionally, to qualitatively evaluate alternative textual meanings of the synset differences, which form the atoms in our hierarchical dictionary, we use CLIP text embeddings and GPT-5 \cite{openai_gpt5}. First, for each parent-child pair, we generate a text description of the difference between the parent and child synsets using GPT-5. Then, we pool the text embeddings of these descriptions to form a set of candidate concept embeddings. Next, for each synset difference vector, we find its top-$k$ neighbors among the candidate concept embeddings. Finally, we use GPT-5 to generate a summary of the top-$k$ neighbor descriptions, which helps interpret the synset difference. \Cref{tab:synset_difference_interpretations} shows example interpretations for several parent-child pairs in WordNet \citep{miller_wordnet_1995}, the hierarchy underlying ImageNet.

\begin{table}[ht]
  \centering
  \caption{Example text interpretations of child--parent synset differences produced via CLIP embeddings and GPT summarization.}
  \label{tab:synset_difference_interpretations}
  \resizebox{\linewidth}{!}{%
  \begin{tabular}{l p{0.5\linewidth}}
    \toprule
    Parent$\rightarrow$Child Pair & Text Interpretation \\
    \midrule
    \texttt{bear} $\rightarrow$ \texttt{polar bear} & thick matte white fur blending with snow. \\
    \texttt{container} $\rightarrow$ \texttt{basket} & open-top woven or perforated sides with handles. \\
    \texttt{structure} $\rightarrow$ \texttt{lumbermill} & vertical log-sawing machines and plank conveyors. \\
    \texttt{citrus} $\rightarrow$ \texttt{orange} & round, bright orange, pebbled rind. \\
    \bottomrule
  \end{tabular}%
  }
  \vspace{-0.55em}
\end{table}

\subsection{Additional Real-Data Experiment Details}\label{app:real_data_experiment_details}

\myparagraph{Model Architecture and Training Details}
We use CLIP-ViT-L/14 as the backbone. To train the linear classifier, we use the AdamW optimizer \citep{loshchilov2017decoupled} with weight decay $10^{-4}$ and learning rate $10^{-1}$. To train the CBM, we use the Adam optimizer with learning rate $10^{-1}$ for $500$ epochs. We provide the detailed hyperparameters in \Cref{tab:hyperparameters}.

\myparagraph{Synset Difference Interpretations}
We use CLIP text embeddings \citep{radford_learning_2021} and GPT-5 \citep{openai_gpt5} to generate textual interpretations of the synset differences. We provide the top-10 concepts for each parent-child pair in \Cref{tab:synset_difference_interpretations_with_concepts}. We also include the GPT-5 prompt in \Cref{tab:residual_phrase_generation_prompt}.

\myparagraph{Ablation Study on Beam Size}
We perform an ablation study on the beam size for Hierarchical OMP. We vary the beam size from 1 to 8 and evaluate concept recovery accuracy on ImageNette. We report the results in \Cref{fig:imagenette_results_ablation_study}.

\myparagraph{Taxonomy Generation Prompt}
We use the taxonomy generation prompt from \citet{zeng2024chain} in \Cref{tab:taxonomy_generation_prompt} to generate the taxonomy for CIFAR-100.

\begin{table}[ht]
  \centering
  \caption{Key hyperparameters used in experiments for each dataset.}
  \label{tab:hyperparameters}
  \resizebox{0.95\linewidth}{!}{%
  \begin{tabular}{lccc}
    \toprule
    \textbf{Hyperparameter} & \textbf{ImageNette} & \textbf{CIFAR-100} & \textbf{ImageNet} \\
    \midrule
    Batch size              & 4096                  & 4096               & 16384                   \\
    Classification training epochs   & 500                   & 500                & 1000                       \\
    \ours{} beam size & 8 & 16 & 32 \\
    \bottomrule
  \end{tabular}
  }
\end{table}

\begin{figure*}[t]
  \centering
  \includegraphics[width=\textwidth]{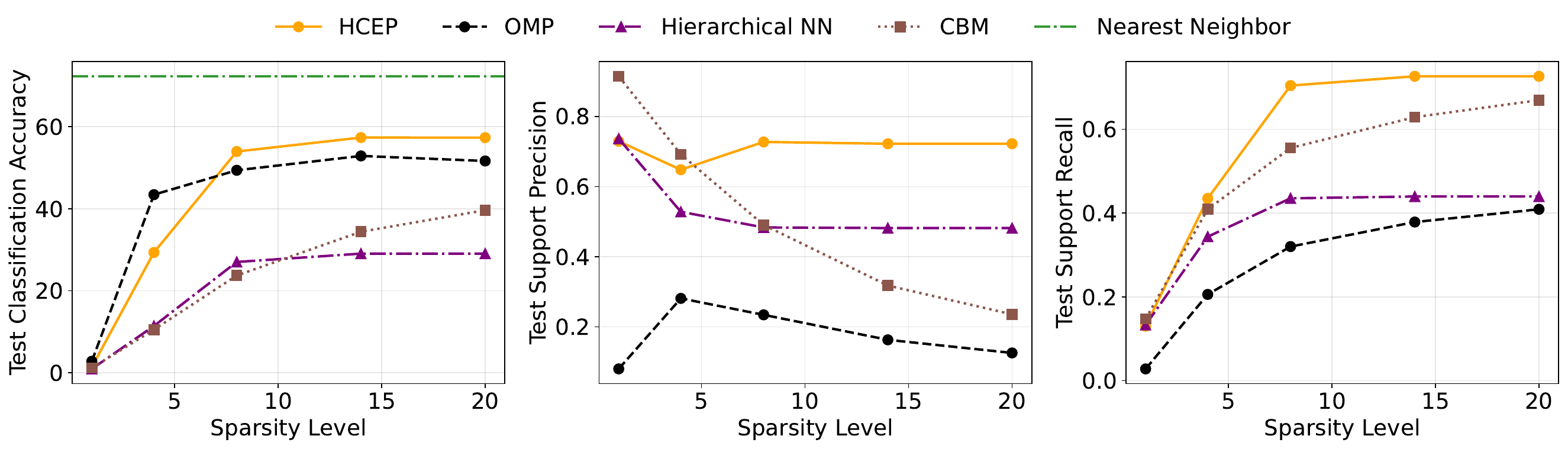}
  \caption{When we restrict the ImageNet training set to 25 images per class, \ours{} outperforms all interpretable baselines.}
  \label{fig:imagenet_results_25_shot}
\end{figure*}

\begin{table*}[ht]
  \centering
  \caption{Example text interpretations of child--parent synset differences produced via CLIP embeddings and GPT summarization, with the top-10 contributing concepts.}
  \label{tab:synset_difference_interpretations_with_concepts}
  \resizebox{\linewidth}{!}{%
  \begin{tabular}{l p{0.35\linewidth} p{0.35\linewidth}}
    \toprule
    Parent$\rightarrow$Child Pair & Top-10 Concepts & Text Interpretation \\
    \midrule
    \texttt{bear} $\rightarrow$ \texttt{polar bear} &
    \begin{enumerate}[leftmargin=*, nosep]
      \item Matte white texture
      \item Thick white winter fur coat
      \item White plumage in winter season
      \item White fur with cream patches
      \item Long, silky white coat
      \item White fur blending with surroundings
      \item Pure white fluffy coat
      \item White coat with lemon markings
      \item White cue ball
      \item Long, corded white coat
    \end{enumerate} &
    thick matte white fur blending with snow. \\
    \midrule
    \texttt{container} $\rightarrow$ \texttt{basket} &
    \begin{enumerate}[leftmargin=*, nosep]
      \item Wicker baskets filled with baguettes
      \item Rectangular shopping baskets
      \item Stacked woven baskets
      \item Rear storage basket
      \item Rectangular open-top basket
      \item Woven rattan backrest
      \item Plastic shopping baskets
      \item Perforated cutlery basket in lower rack
      \item Woven rattan seating surfaces
      \item Rectangular basket frame
    \end{enumerate} &
    open-top woven or perforated sides with handles. \\
    \midrule
    \texttt{structure} $\rightarrow$ \texttt{lumbermill} &
    \begin{enumerate}[leftmargin=*, nosep]
      \item Vertical log slicing machines
      \item Massive log cutting machines
      \item Metallic sawmill machinery
      \item Heavy-duty sawmill frames
      \item Conveyor belts with wood pieces
      \item Wooden board sorting systems
      \item Exposed wooden axles
      \item Wooden log feeding chutes
      \item Stacks of cut wooden planks
      \item Narrow wooden steering wheel
    \end{enumerate} &
    vertical log-sawing machines and plank conveyors. \\
    \bottomrule
  \end{tabular}%
  }
  \vspace{-0.55em}
\end{table*}

\begin{table*}[t]
\caption{LLM prompt used for taxonomy generation from root and leaf concepts.}
\label{tab:taxonomy_generation_prompt}
\begin{tcolorbox}[
  enhanced,
  colback=blue!5!white,
  colframe=blue!50!black,
  title={\large\bfseries Taxonomy Generation Prompt},
  fonttitle=\bfseries,
  top=10pt,
  bottom=10pt,
  left=10pt,
  right=10pt
]

Given root concept \texttt{<root>} and leaf concepts \texttt{<leaves>}, generate a detailed hierarchical taxonomy that organizes these leaves under the root. Create multiple levels of intermediate category hierarchies to build a rich, fine-grained classification structure. Use as many hierarchical levels as needed to create meaningful semantic groupings and subgroupings. 

The format is: 
1. Parent Concept 1.1 Child Concept 1.1.1 Grandchild Concept. 

CRITICAL: Every single leaf concept from the list must appear exactly as given in the taxonomy as the deepest level nodes. You may and should add multiple levels of intermediate concepts but do not add new leaf concepts. Before finishing, verify that each leaf concept from \texttt{<leaves>} appears in your taxonomy. Aim for depth and semantic richness in the hierarchy.

\end{tcolorbox}
\end{table*}

\begin{table*}[t]
\caption{LLM prompt used to generate child-vs-parent residual phrases.}
\label{tab:residual_phrase_generation_prompt}
\begin{tcolorbox}[
  enhanced,
  colback=blue!5!white,
  colframe=blue!50!black,
  title={\large\bfseries Residual Concept Generation Prompt},
  fonttitle=\bfseries,
  top=10pt,
  bottom=10pt,
  left=10pt,
  right=10pt
]
TASK: Generate a concise phrase (3-10 words) that describes what distinguishes "\texttt{<child>}" from its parent category "\texttt{<parent>}".
This is for a hierarchical sparse model. A residual represents the visual difference between a child and parent category.
Most correlated visual concepts (from CLIP embeddings): \texttt{<concepts\_string>}
REQUIREMENTS:

1. Generate ONE short phrase (3-10 words) that captures the key distinguishing features

2. Base your phrase on the correlated concepts provided above

3. Focus on the most salient visual features

4. Be specific and concrete, not vague or generic

5. Use natural language that a human would use to describe the difference

6. IMPORTANT: Do NOT use the synset names ("\texttt{<parent>}" or "\texttt{<child>}") in your phrase

7. IMPORTANT: Describe only the visual features, not the category name

EXAMPLES OF GOOD PHRASES:
- "tawny coat with distinctive facial markings" (for a specific dog breed)
- "long curved neck and pink coloration" (for flamingo vs bird)
- "striped pattern and elongated body" (for a specific fish)
- "metallic surface with cylindrical shape" (for a lighter)
\end{tcolorbox}
\end{table*}

\begin{figure*}[t]
  \centering
  \includegraphics[width=\textwidth]{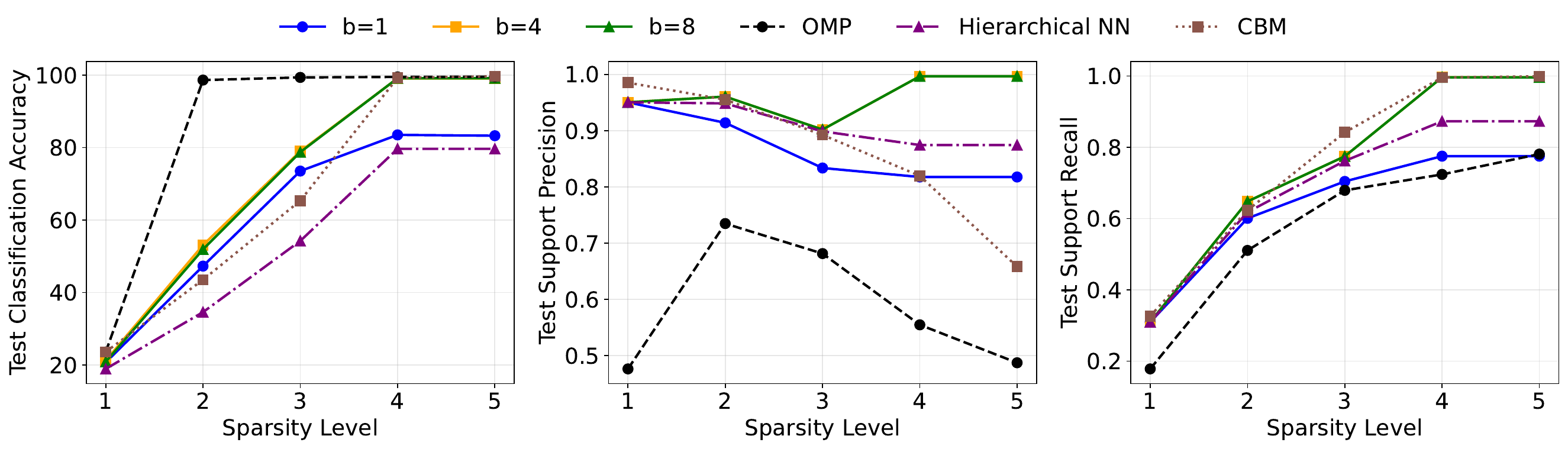}
  \caption{We vary the beam size over \{1, 4, 8\} and evaluate on ImageNette.}
  \label{fig:imagenette_results_ablation_study}
\end{figure*}

\begin{figure*}[t]
  \centering
  \includegraphics[width=0.7\linewidth]{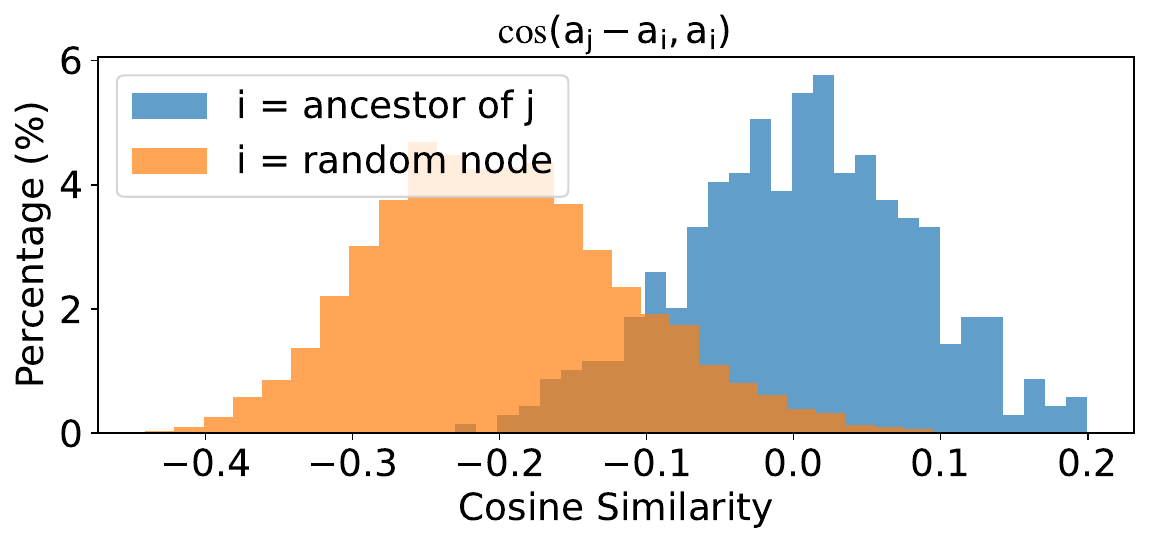}
  \caption{Observed hierarchical orthogonality on CIFAR-100. The cosine similarity between child-parent difference vectors and their parents is close to zero, while random non-parent pairs have significantly lower cosine similarity.}
  \label{fig:orthogonality_test_cifar100}
\end{figure*}

\end{document}